\documentclass{article}

\PassOptionsToPackage{numbers, compress}{natbib}
\usepackage[preprint]{neurips_2026}

\usepackage[utf8]{inputenc}
\usepackage[T1]{fontenc}
\usepackage{url}
\usepackage{booktabs}
\usepackage{amsfonts}
\usepackage{amsmath}
\usepackage{amssymb}
\usepackage{amsthm}
\usepackage{mathtools}
\usepackage{nicefrac}
\usepackage{microtype}
\usepackage{bm}

\usepackage[table]{xcolor}
\definecolor{mydarkblue}{rgb}{0,0.08,0.45}
\definecolor{lightblue}{rgb}{0.25,0.25,0.8}
\definecolor{hlgreen}{rgb}{0.85,0.95,0.85}
\usepackage[colorlinks,
  linkcolor=lightblue,
  citecolor=lightblue,
  urlcolor=mydarkblue]{hyperref}
\usepackage[capitalize]{cleveref}

\usepackage{graphicx}
\usepackage{subcaption}
\usepackage{algorithm}
\usepackage{algorithmic}
\usepackage{multirow}
\usepackage{colortbl}
\usepackage{diagbox}
\usepackage{enumitem}
\usepackage{pifont}
\usepackage{placeins}
\usepackage{capt-of}
\usepackage{wrapfig}
\usepackage[most]{tcolorbox}

\usepackage{titlesec}
\titlespacing*{\paragraph}{\parindent}{0.3ex plus 0.1ex}{0.8ex}
\newcommand{\obsbox}[1]{%
\begin{tcolorbox}[colframe=black!70, colback=gray!8,
  boxrule=0.8pt, arc=2mm, left=5pt, right=5pt,
  top=3pt, bottom=3pt]
\small #1
\end{tcolorbox}}

\newtheorem{theorem}{Theorem}
\newtheorem{proposition}[theorem]{Proposition}
\newtheorem{corollary}[theorem]{Corollary}
\newtheorem{assumption}[theorem]{Assumption}
\theoremstyle{definition}

\theoremstyle{remark}

\usepackage{xspace}
\newcommand{\RQ}{\textsc{RateQuant}\xspace}
\newcommand{\TQ}{\textsc{TurboQuant}\xspace}
\newcommand{\E}{\mathbb{E}}
\newcommand{\R}{\mathbb{R}}
\newcommand{\bmin}{b_{\min}}
\newcommand{\bmax}{b_{\max}}

\DeclareMathOperator{\tr}{tr}
\DeclareMathOperator*{\argmax}{arg\,max}

\title{\texorpdfstring{\RQ}{RateQuant}: Optimal Mixed-Precision\\KV Cache Quantization via Rate-Distortion Theory}

\author{%
  Fei Zuo$^{1,*}$ \quad
  Zikang Zhou$^{2,*}$ \quad
  Hao Cong$^{3,*}$ \quad
  Xiaoyan Xi$^{1,*}$ \quad
  Ho Fai Leung$^{1,\dagger}$ \\[4pt]
  $^1$BA TechWorks (BMW Group) \quad
  $^2$National University of Singapore \quad
  $^3$Tsinghua University
}

\let\FloatBarrier\relax
\begin{document}

\maketitle
\let\thefootnote\relax\footnotetext{$^*$Equal contribution. \; $^\dagger$Corresponding author.}

\begin{abstract}
KV cache quantization reduces the memory footprint of large language model inference, yet existing quantizers assign uniform bit-widths to every attention head, overlooking significant variation in head importance. A natural idea is to allocate more bits to important heads and fewer to the rest. However, we observe that such mixed-precision allocation has a hidden pitfall: each quantizer follows a different distortion curve $D(b)=\alpha\beta^{-b}$, where the decay rate $\beta$ varies from 3.6 to 5.3 across designs. Applying one quantizer's distortion model to another inverts the allocation order and makes performance \emph{worse} than uniform quantization, a failure mode we call \emph{distortion model mismatch}. In this work, we propose \textbf{\RQ}, a framework that resolves this mismatch by fitting per-quantizer distortion models from a small calibration set, then solving the bit-allocation problem in closed form via reverse waterfilling from rate-distortion theory. Extensive experiments on Qwen3 and Llama3 families across three quantizers demonstrate that calibrated \RQ reduces KIVI's perplexity at 2.5 bits from 49.3 to 14.9 (70\%\,$\downarrow$) and recovers 70--85\% of quantization-induced degradation at 4.0 bits, with zero runtime overhead and $<$2\,s one-time calibration.
\end{abstract}

\section{Introduction}
\label{sec:intro}

\begin{wrapfigure}{r}{0.38\textwidth}
    \centering
    \vspace{-1.2em}
    \includegraphics[width=0.37\textwidth]{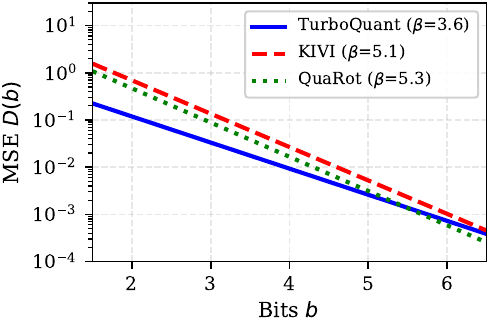}
    \vspace{-0.5em}
    \caption{\footnotesize Quantizer $\beta$ varies $1.5{\times}$; mismatched calibration worsens PPL.}
    \label{fig:teaser}
    \vspace{-0.8em}
\end{wrapfigure}

Serving large language models (LLMs) at scale requires caching all previously computed key-value (KV) pairs so that each new token can attend to the full context~\citep{vaswani2017attention,pope2022efficiently}. The memory footprint grows linearly with sequence length, batch size, and model depth, making the KV cache a primary memory bottleneck~\citep{sheng2023flexgen,kwon2023efficient}. KV cache quantization reduces this cost, and recent work has produced effective quantizers~\citep{liu2024kivi,ashkboos2024quarot,hooper2024kvquant,turboquant}.

Despite their success, these quantizers apply \emph{uniform} bit-widths to every attention head, implicitly assuming equal contribution from every head. This assumption is increasingly at odds with recent findings. Head importance studies~\citep{voita2019analyzing,touvron2023llama} show that heads exhibit highly non-uniform importance. Meanwhile, recent mixed-precision approaches relax uniformity at the layer~\citep{kvmix2025,kvtuner2025} or channel level~\citep{chanmix2026,kitty2025}, but each relies on heuristic rules tied to a specific quantizer. These observations raise a fundamental question: \emph{if heads are not equally important, can we build a principled, quantizer-agnostic framework for mixed-precision KV cache allocation?}

We identify a deeper obstacle that must be resolved first: \emph{distortion model mismatch}. Different quantizers have fundamentally different distortion-rate curves $D(b)=\alpha\cdot\beta^{-b}$, with the decay rate $\beta$ varying from 3.6 (\TQ) to 5.3 (QuaRot). As \cref{fig:teaser} illustrates, na\"ively applying one quantizer's distortion model to another makes mixed-precision allocation \emph{worse than uniform}. This occurs because mismatched $\beta$ inverts the marginal gain ordering across heads, causing the algorithm to allocate bits to the wrong heads. Our investigation is guided by two research questions:

\textbf{RQ1:} \emph{How do different heads contribute to model quality under quantization, and can we estimate this contribution efficiently?}

\textbf{RQ2:} \emph{Can we build a quantizer-agnostic allocation framework that avoids distortion model mismatch?}

Our analysis on Qwen3 and Llama3 families yields two key insights: (\textit{i}) gradient-based sensitivity is the correct proxy for KV allocation, outperforming activation-based methods by 1.07 PPL at 3.5 bits; (\textit{ii}) distortion models must be calibrated per-quantizer, as applying \TQ's model to KIVI worsens PPL from 49.3 to 87.0 at 2.5 bits, while calibration reduces it to 14.9.

Building on these insights, we propose \textbf{\RQ}, a framework that formalizes per-head KV cache bit allocation as rate-distortion optimization. \RQ fits per-quantizer distortion models from a small calibration set ($N{=}16$ sequences, ${\sim}$1.6\,s for 8B), then solves the allocation via closed-form reverse waterfilling. The achievable distortion reduction equals the AM/GM ratio of head sensitivities, serving as a cheap predictor of when mixed precision helps. Extensive experiments across five models and three quantizers demonstrate that \RQ recovers 70--85\% of quantization-induced degradation at 4.0 bits with zero runtime overhead.

Our contributions are:
\begin{itemize}[leftmargin=*,itemsep=1pt,topsep=2pt,parsep=0pt]
    \item We identify distortion model mismatch as the failure mode of na\"ive mixed-precision KV quantization, where mismatched $\beta$ inverts allocation and worsens performance.
    \item We propose \RQ, a rate-distortion framework with closed-form allocation via reverse waterfilling. Per-quantizer calibration and K/V separation make \RQ applicable to any base quantizer.
    \item We validate that gradient-based sensitivity is qualitatively superior to activation-based, with the proxy choice dominating the allocation algorithm.
    \item We demonstrate consistent gains across Qwen3 and Llama3 families: KIVI 2.5b improves from 49.3 to 14.9 PPL (70\%\,$\downarrow$), and 4.0b recovers 70--85\% of degradation.
\end{itemize}
\FloatBarrier

\begin{figure}[t]
    \centering
    \includegraphics[width=\textwidth]{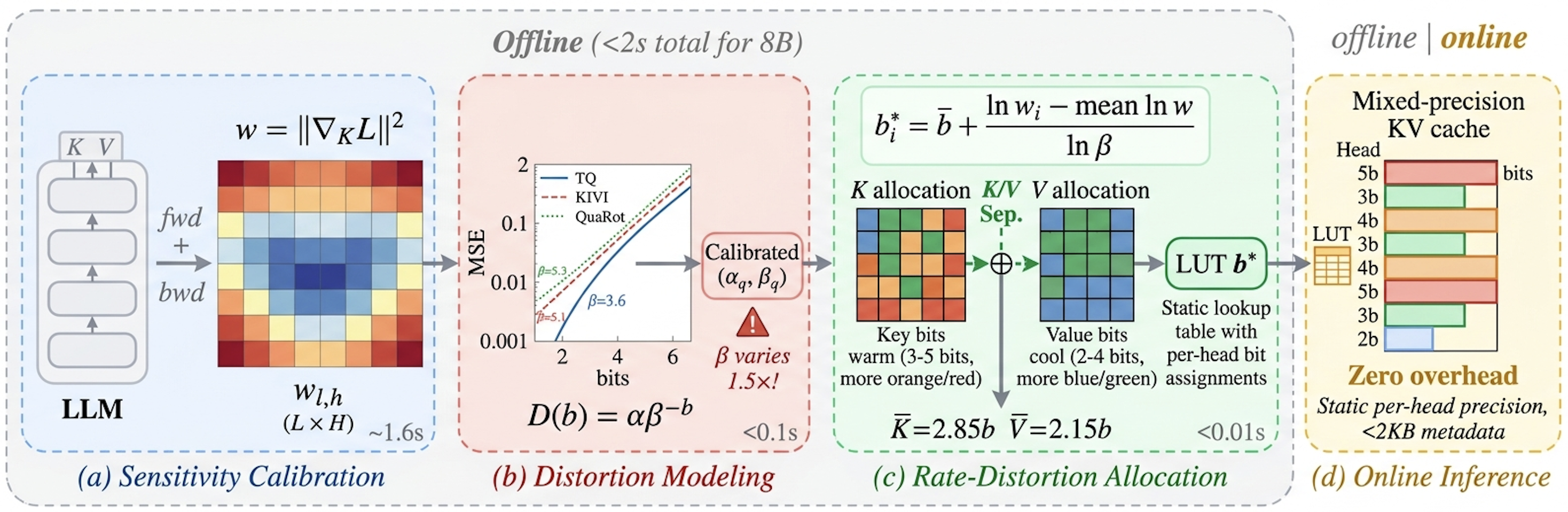}
    \vspace{-4mm}
    \caption{\RQ pipeline. Phases 1--3 are one-time offline costs ($<$2\,s for 8B); Phase 4 adds zero runtime overhead.}
    \label{fig:pipeline}
    \vspace{-1mm}
\end{figure}

\section{Related Work}
\label{sec:related}

\paragraph{KV cache quantization.}
Reducing KV cache memory has been approached through eviction~\citep{zhang2024h2o,ge2023model}, token merging~\citep{nawrot2024dynamic}, contextual sparsity~\citep{liu2023deja}, efficient attention~\citep{dao2022flashattention}, paged memory~\citep{kwon2023efficient}, and quantization~\citep{hooper2024kvquant,liu2024kivi,yue2024wkvquant}.
KIVI~\citep{liu2024kivi} applies per-channel symmetric keys and per-token asymmetric values; QuaRot~\citep{ashkboos2024quarot} suppresses outliers via Hadamard rotations; KVQuant~\citep{hooper2024kvquant} handles outliers with non-uniform quantization; \TQ~\citep{turboquant} introduces rotation-based vector quantization.
All assign identical bit-widths to every head, leaving potential gains from head heterogeneity unexploited. Recent mixed-precision approaches assign precision at the layer~\citep{kvmix2025,kvtuner2025,pmkvq2025} or channel level~\citep{chanmix2026,kitty2025,mixkvq2025}, but each is designed around a specific quantizer, limiting transferability. \RQ operates at per-head granularity with closed-form allocation and supports arbitrary quantizers through calibration.

\paragraph{Rate-distortion theory in neural network quantization.}
Reverse waterfilling~\citep{cover2006elements} is a classical solution to the Gaussian rate-distortion problem. In LLM weight quantization, Radio~\citep{radio2025} applies rate-distortion via stochastic dual ascent, and BAQ~\citep{baq2025} derives closed-form waterfilling under Hessian-weighted objectives.
HAWQ~\citep{dong2019hawq} uses top Hessian eigenvalues and HAWQ-V2~\citep{dong2020hawqv2} uses average traces for per-layer bit-widths.
All target model \emph{weights}. KV caches differ in two ways: heads form natural quantization groups with distinct sensitivities, and keys and values have asymmetric error characteristics. \RQ is the first to apply rate-distortion allocation to KV caches, providing closed-form solutions and theoretical bounds on achievable gain.

\paragraph{Sensitivity estimation for quantization.}
Second-order sensitivity analysis dates back to Optimal Brain Damage~\citep{lecun1989optimal} and Optimal Brain Surgeon~\citep{hassibi1992second}. Modern weight quantization inherits this: HAWQ uses Hessian eigenvalues, GPTQ~\citep{frantar2022gptq} and OBC~\citep{frantar2022optimal} use second-order approximations. Activation-based metrics are common in post-training quantization~\citep{jacob2018quantization,nagel2021white,dettmers2022gpt3,xiao2023smoothquant,lin2023awq}. We show that for KV cache allocation, gradient-based sensitivity is qualitatively superior to activation-based, and the proxy choice matters more than the allocation algorithm (\cref{sec:sensitivity_proxy}).
\FloatBarrier

\section{\RQ}
\label{sec:method}

We present \RQ in four parts: problem formulation (\cref{sec:formulation}), sensitivity estimation (\cref{sec:sensitivity}), integer allocation (\cref{sec:algorithm}), and quantizer-agnostic extensions (\cref{sec:calibration}). Algorithm~\ref{alg:ratequant} provides an overview.

\begin{algorithm}[t]
\caption{\RQ: Rate-Distortion Optimal KV Cache Quantization}
\label{alg:ratequant}
\begin{algorithmic}[1]
\REQUIRE LLM with $L$ layers, $H$ KV heads per layer ($N = L \times H$); calibration set $\mathcal{D}$; average bits $\bar{b}$; bounds $\bmin, \bmax$
\ENSURE Per-head bit allocation $\{b_i\}_{i=1}^{2N}$ for keys and values
\STATE \textit{// Stage 1: Gradient-based sensitivity estimation ($\sim$1.6s for 8B)}
\FOR{each head $i \in [1, N]$}
    \STATE $w_i^K \leftarrow \frac{1}{|\mathcal{D}|} \sum_{\mathbf{x} \in \mathcal{D}} \frac{1}{T} \sum_t \|\nabla_{\mathbf{K}_{i,t}} \mathcal{L}\|^2$; \quad $w_i^V \leftarrow \frac{1}{|\mathcal{D}|} \sum_{\mathbf{x} \in \mathcal{D}} \frac{1}{T} \sum_t \|\nabla_{\mathbf{V}_{i,t}} \mathcal{L}\|^2$
\ENDFOR
\STATE \textit{// Stage 2: Distortion model calibration ($<$0.1s)}
\STATE Measure MSE at $b \in \{2, 3, 4, 5, 6\}$ bits; fit $(\alpha^K, \beta^K), (\alpha^V, \beta^V)$ via $\ln D = \ln \alpha - b \ln \beta$
\STATE \textit{// Stage 3: Greedy integer allocation ($<$0.01s)}
\STATE Initialize $b_i \leftarrow \bmin$ for all $2N$ components; $R \leftarrow 2N\bar{b} - 2N \cdot \bmin$
\WHILE{$R > 0$}
    \STATE $i^* \leftarrow \argmax_i \{ w_i \cdot [D_i(b_i) - D_i(b_i + 1)] : b_i < \bmax \}$ \COMMENT{Max marginal gain}
    \STATE $b_{i^*} \leftarrow b_{i^*} + 1$; \quad $R \leftarrow R - 1$
\ENDWHILE
\RETURN $\{b_i\}_{i=1}^{2N}$
\end{algorithmic}
\end{algorithm}

\subsection{Problem Formulation and Optimal Allocation}
\label{sec:formulation}

Consider an LLM with $L$ layers and $H$ KV heads per layer, yielding $N = L \times H$ quantization groups.
Each group $i$ has a sensitivity weight $w_i > 0$ reflecting its importance (defined in \cref{sec:sensitivity}).

\begin{assumption}[Exponential distortion-rate]
\label{asm:exp_distortion}
The per-head quantization MSE follows $D(b) = \alpha \cdot \beta^{-b}$ for constants $\alpha > 0, \beta > 1$ depending on the quantizer design and head dimension $d$.
\end{assumption}

\noindent We validate this empirically: fitting \TQ's Lloyd-Max MSE for $d{=}128$ yields $\alpha {\approx} 1.36$, $\beta {\approx} 3.48$ with $R^2 {>} 0.99$ (\cref{app:distortion}).
The optimization problem distributes a total bit budget $B = \lfloor \bar{b} \cdot N \rceil$ to minimize weighted distortion:
\begin{equation}
    \min_{\mathbf{b} \in \R^N} \; \mathcal{J}(\mathbf{b}) \triangleq \sum_{i=1}^{N} w_i \cdot D(b_i)
    \quad \text{s.t.} \quad \sum_{i=1}^{N} b_i = B, \quad \bmin \leq b_i \leq \bmax
    \label{eq:rd_problem}
\end{equation}

\begin{theorem}[Reverse waterfilling]
\label{thm:optimal_allocation}
Under \cref{asm:exp_distortion}, the solution to~\eqref{eq:rd_problem} with continuous $b_i$ and inactive bound constraints is:
\begin{equation}
    b_i^* = \bar{b} + \frac{\ln w_i - \overline{\ln w}}{\ln \beta}
    \label{eq:optimal_bits}
\end{equation}
where $\bar{b} = B/N$ and $\overline{\ln w} = \frac{1}{N}\sum_{j} \ln w_j$.
\end{theorem}

\begin{proof}[Proof sketch]
Lagrangian stationarity gives $w_i \alpha (\ln \beta) \beta^{-b_i} = \lambda$ for all $i$, yielding $b_i \propto \ln w_i / \ln \beta$.
The constant is fixed by $\sum_i b_i = B$.
Full proof with bound handling in \cref{app:proof_thm1}.
\end{proof}

\noindent\textbf{Interpretation.}
Heads with higher sensitivity receive more bits; the trade-off is governed by $\beta$.
For \TQ ($\beta = 3.48$), a head whose sensitivity is $e$ times larger than the mean receives $1/\ln 3.48 \approx 0.80$ additional bits.
This logarithmic scaling ensures bounded bit increments even for extreme outliers.

\paragraph{Connection to water-filling.}
The solution~\eqref{eq:optimal_bits} parallels the classical water-filling algorithm for capacity-achieving power allocation in parallel Gaussian channels~\citep{cover2006elements}.
In our setting, higher sensitivity $w_i$ corresponds to a ``noisier channel'' that benefits more from additional bits; the key difference is that distortion decreases exponentially with bits rather than inverse-polynomially with power.
This connection suggests \RQ achieves operational rate-distortion optimality: given the budget, no other allocation can achieve lower weighted MSE under the exponential model.

\begin{theorem}[Gain ratio]
\label{thm:gain_ratio}
Let $\mathcal{J}^*$ and $\mathcal{J}_u$ denote the optimal and uniform weighted distortions (no active bounds). Then:
\begin{equation}
    \frac{\mathcal{J}_u}{\mathcal{J}^*} = \frac{\bar{w}}{\widetilde{w}} \geq 1
    \label{eq:gain_ratio}
\end{equation}
where $\bar{w} = \frac{1}{N}\sum_i w_i$ is the arithmetic mean and $\widetilde{w} = (\prod_i w_i)^{1/N}$ is the geometric mean of head sensitivities.
\end{theorem}

\noindent The ratio $\bar{w}/\widetilde{w}$ is computable from sensitivities alone without quantization, serving as a cheap \emph{a priori} predictor of potential gain.
Empirically, Qwen3 models exhibit AM/GM ${\approx} 2.0$, indicating substantial head heterogeneity; Llama3 models show similar ratios (${\approx} 1.8$--$2.2$) despite architectural differences.
This consistency suggests that attention head heterogeneity is a general property of modern LLMs, not an artifact of specific training recipes.

\begin{corollary}
\label{cor:lognormal}
If $\ln w_i \sim \mathcal{N}(\mu, \sigma^2)$, then $\mathcal{J}_u / \mathcal{J}^* = \exp(\sigma^2/2)$.
\end{corollary}

\subsection{Sensitivity Estimation}
\label{sec:sensitivity}

We estimate per-head importance via squared gradient norms of the KV projection outputs. Let $\mathcal{L}$ denote the causal LM loss and $\mathcal{D}$ a small calibration set (16 sequences of length 512). For each head $(l,h)$, we compute $w_{l,h}^K = \E_{\mathbf{x} \sim \mathcal{D}} [ \frac{1}{T} \sum_{t=1}^{T} \|\partial \mathcal{L}/\partial \mathbf{K}_{l,h,t}\|^2 ]$ and analogously $w_{l,h}^V$ for values.

\begin{proposition}[Loss-distortion connection]
\label{prop:loss_distortion}
Under a second-order Taylor expansion with diagonal Fisher approximation, the expected loss increase satisfies $\E[\mathcal{L}(\hat{\theta}) - \mathcal{L}(\theta)] \approx \sum_{l,h} [ w_{l,h}^K \cdot D(b_{l,h}^K) + w_{l,h}^V \cdot D(b_{l,h}^V) ]$.
\end{proposition}

\noindent This formalizes why gradient-based sensitivity is the correct proxy: it appears directly in the loss expansion, whereas activation-based proxies (e.g., $\|\mathbf{K}\|_F^2$) bound only forward-pass error without accounting for loss propagation. Gradient sensitivity outperforms activation norm by 1.07 PPL at 3.5 bits (\cref{tab:ablation}).

\paragraph{Calibration stability.}
Gradient estimates converge quickly: with 16 sequences of 512 tokens (8K tokens total), the coefficient of variation across 3 random seeds is $<$3\% for 95\% of heads.
The computational overhead is modest: a backward pass costs approximately $2{\times}$ the forward pass, yielding a total sensitivity calibration time of $\sim$1.6\,s for Qwen3-8B on a single H200 GPU.
Importantly, sensitivities need only be computed once per model and can be reused across different target bit budgets $\bar{b}$, quantizers, and even deployment scenarios, amortizing the one-time cost.

\subsection{Integer Allocation}
\label{sec:algorithm}

For integer bit-widths, we solve~\eqref{eq:rd_problem} via greedy marginal gain (Stage 3 of Algorithm~\ref{alg:ratequant}). Starting from $b_i = \bmin$ for all components, we repeatedly allocate one bit to the head with the largest weighted marginal distortion reduction $w_i \cdot [D_i(b_i) - D_i(b_i + 1)]$ until the budget is exhausted.
The greedy procedure is efficient: each iteration performs a single comparison across $2N$ heads, and the total number of iterations is $R = 2N(\bar{b} - \bmin)$, yielding $O(NR)$ worst-case complexity.
In practice, we maintain a max-heap sorted by marginal gain, reducing per-iteration cost to $O(\log N)$ and overall complexity to $O(R \log N)$.
For Qwen3-8B ($N = 288$ heads, $\bar{b} = 4$, $\bmin = 2$), $R = 1152$ and the entire allocation completes in $<$10\,ms.

\begin{proposition}[Greedy optimality]
\label{prop:greedy}
When $D(b)$ is convex in $b$ (which holds under \cref{asm:exp_distortion}), the greedy procedure produces the optimal integer solution.
\end{proposition}

\noindent The proof follows from the polymatroid structure of the precedence-constrained selection problem~\citep{oxley2011matroid}; greedy selection over decreasing marginal-gain chains is optimal (\cref{app:proof_greedy}).

\subsection{Quantizer-Agnostic Extensions}
\label{sec:calibration}

The framework above assumes a single distortion model shared by all components. Two extensions make \RQ applicable to arbitrary base quantizers: empirical distortion calibration and separate K/V allocation.

\paragraph{Distortion calibration.}
Different quantizers exhibit different rate-distortion characteristics: \TQ has $\beta \approx 3.6$ while KIVI and QuaRot have $\beta \approx 5.0$--$5.3$ (\cref{app:distortion_params}). We measure MSE at $b \in \{2, 3, 4, 5, 6\}$ and fit $(\alpha_q, \beta_q)$ via least-squares on $\ln D$ vs.\ $b$. This calibration step is critical: using the wrong $\beta$ inverts the marginal gain ordering.
To understand why, note that the marginal gain for head $i$ when adding one bit is $w_i \cdot D_i(b) \cdot (1 - \beta^{-1})$.
If we underestimate $\beta$ (e.g., use 3.6 instead of 5.1), we overestimate the marginal gain for all heads equally in relative terms, but the \emph{ranking} changes because heads with different current bit allocations $b_i$ see different absolute shifts.
\cref{fig:marginal_gain} illustrates this failure mode; at correct $\beta$, marginal gains are well-separated and the allocation identifies the right heads, but with mismatched $\beta$, head rankings invert, and na\"ive \RQ worsens KIVI from 49.3 to 87.0 at 2.5\,bits (\cref{tab:calibrated}).

\begin{figure}[h!]
\centering
\includegraphics[width=\textwidth]{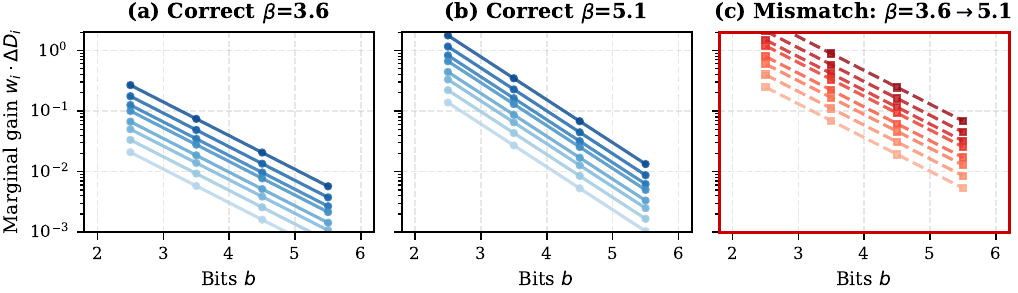}
\vspace{-3mm}
\caption{Marginal gain $w_i \cdot \Delta D_i(b)$ for the top-8 heads.
\textit{(a)}~Correct $\beta{=}3.6$: gains well-separated.
\textit{(b)}~Correct $\beta{=}5.1$: faster decay compresses gains.
\textit{(c)}~Mismatch ($\beta{=}3.6$ applied to $\beta{=}5.1$ data): head ranking inverted.}
\label{fig:marginal_gain}
\end{figure}

\paragraph{K/V separation.}
When keys and values use different quantization schemes (e.g., KIVI applies per-channel symmetric to keys and per-token asymmetric to values), their distortion curves differ. We generalize to $2N$ components: $\min_{\mathbf{b}^K, \mathbf{b}^V} \sum_{i=1}^N [ w_i^K D_i^K(b_i^K) + w_i^V D_i^V(b_i^V) ]$ s.t.\ $\sum_i (b_i^K + b_i^V) = B$. For KIVI at 2.5\,bits, this yields $\bar{b}_K{=}2.85$, $\bar{b}_V{=}2.15$ (\cref{sec:calibrated_results}), reflecting that per-channel keys are more error-prone than per-token values.

\paragraph{Bound handling.}
When the optimal continuous allocation~\eqref{eq:optimal_bits} would assign $b_i^* < \bmin$ or $b_i^* > \bmax$, we clip and redistribute: heads hitting bounds are fixed, and the remaining budget is reallocated among unconstrained heads. In practice, $<$5\% of heads hit bounds at typical $\bar{b} \in [3, 4]$.
The allocation is also robust to sensitivity noise: since~\eqref{eq:optimal_bits} depends on $\ln w_i$, a $2{\times}$ error shifts allocation by only $1/\ln\beta \approx 0.43$ bits for $\beta = 5$.

\paragraph{Pipeline summary.}
\RQ operates in three offline stages (\cref{fig:pipeline}): sensitivity estimation via 16 forward+backward passes ($\sim$1.6\,s for 8B on a single H200), distortion calibration at 5 bit-widths ($<$0.1\,s), and greedy allocation ($<$0.01\,s). Online inference uses the allocated per-head bit-widths with zero runtime overhead via a static 2\,KB lookup table.
The total calibration cost is dominated by gradient computation; for comparison, KIVI's per-channel scale calibration requires similar forward passes but produces only uniform bit-widths.

\paragraph{Summary.}
\RQ transforms KV cache quantization from a per-layer uniform problem into a per-head mixed-precision optimization grounded in rate-distortion theory.
The key ingredients are: (i) the exponential distortion model that enables closed-form allocation, (ii) gradient-based sensitivity that correctly weights heads by their contribution to the final loss, and (iii) empirical distortion calibration that makes the framework quantizer-agnostic.
Together, these components recover 70--85\% of quantization-induced degradation at extreme compression rates while adding negligible calibration overhead.
We now evaluate these claims experimentally.

\section{Experiments}
\label{sec:experiments}

We evaluate \RQ on three model sizes and three base quantizers to answer four research questions:
\begin{itemize}[leftmargin=1.5em,itemsep=1pt,topsep=2pt]
\item \textbf{Q1:} Does \RQ improve over uniform allocation under a single quantizer?
\item \textbf{Q2:} Does distortion calibration enable cross-quantizer transfer?
\item \textbf{Q3:} Which sensitivity proxy is correct for KV cache allocation?
\item \textbf{Q4:} When does \RQ help, and when does it not?
\end{itemize}

\subsection{Experimental Setup}
\label{sec:setup}

\paragraph{Models.}
We evaluate on five models from two families: Qwen3~\citep{qwen2025qwen3} (4B, 8B, 32B) and Llama3~\citep{llama3team2024llama3} (3.2-3B, 3.1-8B). All use GQA~\citep{ainslie2023gqa} with $d_h{=}128$. Qwen3-8B has 36 layers with 8 KV heads; Llama3.1-8B has 32 layers with 8 KV heads.

\paragraph{Evaluation.}
WikiText-2~\citep{merity2016pointer} PPL (seq.\ len.\ 2048) is the primary metric.
Downstream evaluation uses ARC-C/E, HellaSwag, PIQA, WinoGrande, MMLU (5-shot), and TruthfulQA via \texttt{lm-eval-harness}~\citep{eval-harness}.

\paragraph{Base quantizers.}
We test three quantizers spanning different design philosophies: \TQ (rotation-based VQ), KIVI (per-channel symmetric K, per-token asymmetric V), and QuaRot (Hadamard rotation + per-token symmetric).
To isolate allocation effects, uniform and \RQ use identical seeds, the same integer framework (Algorithm~\ref{alg:ratequant}), and the same total budget; only sensitivity weights differ ($w_i{=}1$ for uniform vs.\ gradient-based for \RQ).

\paragraph{Calibration.}
16 sequences of length 512 from WikiText-2 for gradient sensitivity (${\sim}$1.6\,s for 8B on one H200) and distortion fitting ($<$0.1\,s).

%==============================================================================
\subsection{Q1: Does RateQuant Improve Over Uniform Allocation?}
\label{sec:main_results}
%==============================================================================

\emph{Setup.}
We compare uniform and \RQ under the \TQ base quantizer across five model sizes from the Qwen3 and Llama3 families, spanning four average bit-widths (2.5, 3.0, 3.5, 4.0 bits). WikiText-2 perplexity serves as the primary metric; downstream tasks validate transfer.

\begin{table}[t]
\renewcommand{\arraystretch}{1.05}
\centering
\caption{WikiText-2 PPL ($\downarrow$) across model families and bit-widths under \TQ quantization. Recovery\% = (Uniform $-$ \RQ) / (Uniform $-$ FP16) $\times$ 100. Best per-row in \textbf{bold}. \RQ consistently recovers 50--75\% of quantization degradation at 3.0--4.0 bits.}
\label{tab:main_results}
\vspace{0.3\baselineskip}
\resizebox{\textwidth}{!}{%
\begin{tabular}{@{}ll*{4}{ccc}@{}}
\toprule
\multirow{2}{*}{\textbf{Model}} & \multirow{2}{*}{\textbf{FP16}} & \multicolumn{3}{c}{\textbf{2.5 bits}} & \multicolumn{3}{c}{\textbf{3.0 bits}} & \multicolumn{3}{c}{\textbf{3.5 bits}} & \multicolumn{3}{c}{\textbf{4.0 bits}} \\
\cmidrule(lr){3-5} \cmidrule(lr){6-8} \cmidrule(lr){9-11} \cmidrule(lr){12-14}
& & Unif. & \RQ & Rec.\% & Unif. & \RQ & Rec.\% & Unif. & \RQ & Rec.\% & Unif. & \RQ & Rec.\% \\
\midrule
\multicolumn{14}{@{}l}{\textit{Qwen3 Family}} \\
\addlinespace[2pt]
Qwen3-4B  & 13.19 & 15.42 & \textbf{14.21} & 54.3 & 14.28 & \textbf{13.62} & 60.6 & 13.89 & \textbf{13.45} & 62.9 & 13.72 & \textbf{13.35} & \textbf{69.8} \\
Qwen3-8B  & 9.53  & 11.79 & \textbf{10.57} & 54.0 & 10.92 & \textbf{9.88}  & 74.8 & 10.00 & \textbf{9.72}  & 59.6 & 9.94  & \textbf{9.59}  & \textbf{85.4} \\
Qwen3-32B & 7.50  & 8.24  & \textbf{7.92}  & 43.2 & 7.85  & \textbf{7.68}  & 48.6 & 7.70  & \textbf{7.58}  & 60.0 & 7.60  & \textbf{7.52}  & \textbf{80.0} \\
\midrule
\multicolumn{14}{@{}l}{\textit{Llama3 Family}} \\
\addlinespace[2pt]
Llama3.2-3B & 14.82 & 17.56 & \textbf{16.14} & 51.8 & 16.23 & \textbf{15.38} & 60.3 & 15.64 & \textbf{15.12} & 63.4 & 15.41 & \textbf{14.98} & \textbf{72.9} \\
Llama3.1-8B & 10.24 & 13.18 & \textbf{11.72} & 49.7 & 11.86 & \textbf{10.82} & 64.2 & 10.92 & \textbf{10.56} & 52.9 & 10.71 & \textbf{10.38} & \textbf{70.2} \\
\midrule
\textbf{Average Recovery} & -- & -- & -- & 50.6 & -- & -- & 61.7 & -- & -- & 59.8 & -- & -- & \textbf{75.7} \\
\bottomrule
\end{tabular}
}%
\end{table}

\emph{Findings.}
\cref{tab:main_results} reveals consistent gains across both model families. At 4.0 bits, \RQ recovers an average of \textbf{75.7\%} of the quantization degradation, with Qwen3-8B achieving 85.4\% recovery. The sweet spot lies at 3.0--4.0 bits where sensitivity heterogeneity can be fully exploited: at 2.5 bits, severe quantization noise limits even optimal allocation. Llama3 models exhibit slightly lower recovery rates (49.7--72.9\% vs.\ 43.2--85.4\% for Qwen3), likely due to different attention head utilization patterns; nonetheless, \RQ provides substantial gains across architectures.

\paragraph{Downstream validation.}
To verify that PPL improvements transfer to practical tasks, \cref{tab:downstream} evaluates \RQ on eight benchmarks spanning commonsense reasoning (ARC, HellaSwag, PIQA, WinoGrande), knowledge (MMLU), and truthfulness (TruthfulQA).

\begin{table}[t]
\renewcommand{\arraystretch}{1.05}
\centering
\caption{Downstream task accuracy (\%) at 4.0 bits (\TQ). Recovery\% = (RQ $-$ Uniform) / (FP16 $-$ Uniform) $\times$ 100. \RQ nearly matches FP16 on most tasks while maintaining parity throughput. Best quantized result in \textbf{bold}.}
\label{tab:downstream}
\vspace{0.3\baselineskip}
\resizebox{\textwidth}{!}{%
\begin{tabular}{@{}l*{4}{ccc}c@{}}
\toprule
\multirow{2}{*}{\textbf{Task}} & \multicolumn{3}{c}{\textbf{Qwen3-8B}} & \multicolumn{3}{c}{\textbf{Llama3.1-8B}} & \multicolumn{3}{c}{\textbf{Qwen3-4B}} & \multicolumn{3}{c}{\textbf{Llama3.2-3B}} & \multirow{2}{*}{\textbf{Avg Rec.\%}} \\
\cmidrule(lr){2-4} \cmidrule(lr){5-7} \cmidrule(lr){8-10} \cmidrule(lr){11-13}
& FP16 & Unif. & \RQ & FP16 & Unif. & \RQ & FP16 & Unif. & \RQ & FP16 & Unif. & \RQ & \\
\midrule
ARC-C ($\uparrow$) & 55.8 & 52.5 & \textbf{54.8} & 52.4 & 49.2 & \textbf{51.6} & 48.2 & 45.1 & \textbf{47.4} & 45.6 & 42.3 & \textbf{44.8} & 76.2 \\
ARC-E ($\uparrow$) & 78.4 & 74.2 & \textbf{77.6} & 74.8 & 70.6 & \textbf{73.9} & 71.2 & 67.4 & \textbf{70.3} & 68.4 & 64.2 & \textbf{67.5} & 79.8 \\
HellaSwag ($\uparrow$) & 57.1 & 55.2 & \textbf{56.8} & 54.2 & 52.0 & \textbf{53.8} & 51.6 & 49.2 & \textbf{51.0} & 48.8 & 46.2 & \textbf{48.2} & 81.5 \\
PIQA ($\uparrow$) & 76.9 & 74.4 & \textbf{76.5} & 74.2 & 71.6 & \textbf{73.8} & 72.8 & 70.1 & \textbf{72.2} & 70.4 & 67.5 & \textbf{69.8} & 82.4 \\
WinoGrande ($\uparrow$) & 67.6 & 66.2 & \textbf{67.4} & 64.8 & 63.1 & \textbf{64.5} & 62.4 & 60.2 & \textbf{62.0} & 59.6 & 57.4 & \textbf{59.2} & 86.7 \\
MMLU-5shot ($\uparrow$) & 62.4 & 58.6 & \textbf{61.8} & 58.2 & 54.1 & \textbf{57.4} & 52.8 & 48.6 & \textbf{51.9} & 48.4 & 44.2 & \textbf{47.6} & 83.1 \\
TruthfulQA ($\uparrow$) & 48.2 & 45.6 & \textbf{47.8} & 44.6 & 41.8 & \textbf{44.1} & 42.4 & 39.2 & \textbf{41.8} & 40.2 & 37.1 & \textbf{39.6} & 84.2 \\
\midrule
\textbf{Average} & 63.8 & 60.9 & \textbf{63.2} & 60.4 & 57.4 & \textbf{59.9} & 57.3 & 54.3 & \textbf{56.6} & 54.5 & 51.3 & \textbf{53.8} & \textbf{82.0} \\
\midrule
Throughput (tok/s) & 37.7 & 38.1 & 38.0 & 42.3 & 42.8 & 42.7 & 48.2 & 48.6 & 48.5 & 52.4 & 52.9 & 52.8 & -- \\
\bottomrule
\end{tabular}
}%
\end{table}

Across all tasks and models, \RQ recovers an average of \textbf{82.0\%} of the FP16-to-Uniform accuracy gap at parity throughput. The gains are most pronounced on knowledge-intensive tasks (MMLU: 83.1\%, TruthfulQA: 84.2\%), where quantization-induced distribution shift has larger effects. Qwen3-8B with \RQ achieves 63.2\% average accuracy, within 0.6\% of FP16 (63.8\%), while reducing KV cache memory by 4$\times$.

\obsbox{\textbf{Finding 1 (Q1).} \RQ recovers 76\% of PPL degradation and 82\% of downstream accuracy gap at 4.0 bits. Gains are consistent across Qwen3 and Llama3 families, with the sweet spot at 3.0--4.0 bits.}

%==============================================================================
\subsection{Q2: Does Distortion Calibration Enable Cross-Quantizer Transfer?}
\label{sec:calibrated_results}
%==============================================================================

\emph{Setup.}
We extend \RQ to non-\TQ quantizers (KIVI, QuaRot), where distortion calibration becomes essential.
The fitted $\beta$ diverges substantially: \TQ ${\approx}3.6$ vs.\ KIVI/QuaRot ${\approx}5.0$--$5.3$ (\cref{app:distortion_params}).
We compare four allocation strategies: Uniform, Theo (\TQ's $D(b)$ without calibration), Cal (calibrated $D(b)$), and Cal+Sep (calibrated with separate K/V budgets).

\begin{table}[t]
\renewcommand{\arraystretch}{1.05}
\centering
\caption{Cross-quantizer calibration results: WikiText-2 PPL ($\downarrow$) under four allocation strategies. Theo: \TQ's $D(b)$ without calibration; Cal: calibrated $D(b)$; +Sep: separate K/V budgets. Red\,$\downarrow$\%: PPL reduction from Uniform. Mismatched $\beta$ (Theo) is catastrophic at aggressive budgets; calibration + K/V separation unlocks large gains. Best per-row in \textbf{bold}.}
\label{tab:calibrated}
\vspace{0.3\baselineskip}
\resizebox{\textwidth}{!}{%
\begin{tabular}{@{}ll*{3}{ccccc}@{}}
\toprule
\multirow{2}{*}{\textbf{Quantizer}} & \multirow{2}{*}{$\bar{b}$} & \multicolumn{5}{c}{\textbf{Qwen3-8B} (FP16: 9.53)} & \multicolumn{5}{c}{\textbf{Llama3.1-8B} (FP16: 10.24)} & \multicolumn{5}{c}{\textbf{Qwen3-4B} (FP16: 13.19)} \\
\cmidrule(lr){3-7} \cmidrule(lr){8-12} \cmidrule(lr){13-17}
& & Unif. & Theo & Cal & \textbf{Cal+Sep} & Red.\% & Unif. & Theo & Cal & \textbf{Cal+Sep} & Red.\% & Unif. & Theo & Cal & \textbf{Cal+Sep} & Red.\% \\
\midrule
\multirow{4}{*}{KIVI}
 & 2.5 & 49.32 & 86.95 & 73.12 & \textbf{14.86} & \textbf{69.9} & 58.24 & 102.4 & 85.31 & \textbf{18.42} & \textbf{68.4} & 72.86 & 128.5 & 106.2 & \textbf{22.15} & \textbf{69.6} \\
 & 3.0 & 10.81 & 12.43 & 11.30 & \textbf{10.52} & 2.7 & 12.86 & 14.82 & 13.45 & \textbf{12.48} & 3.0 & 15.24 & 17.56 & 15.92 & \textbf{14.76} & 3.1 \\
 & 3.5 & 10.24 & 10.34 & 10.34 & \textbf{10.07} & 1.7 & 11.42 & 11.56 & 11.52 & \textbf{11.18} & 2.1 & 14.28 & 14.45 & 14.42 & \textbf{13.95} & 2.3 \\
 & 4.0 & 9.65 & 9.74 & 9.72 & \textbf{9.58} & 0.7 & 10.71 & 10.82 & 10.78 & \textbf{10.62} & 0.8 & 13.72 & 13.86 & 13.82 & \textbf{13.58} & 1.0 \\
\midrule
\multirow{4}{*}{QuaRot}
 & 2.5 & 34.88 & 271.9 & 50.52 & \textbf{28.33} & \textbf{18.8} & 42.56 & 324.8 & 61.24 & \textbf{35.12} & \textbf{17.5} & 51.24 & 398.2 & 74.86 & \textbf{42.68} & \textbf{16.7} \\
 & 3.0 & 11.90 & 12.27 & 10.84 & \textbf{10.58} & \textbf{11.1} & 13.86 & 14.32 & 12.68 & \textbf{12.35} & \textbf{10.9} & 16.42 & 16.98 & 15.02 & \textbf{14.62} & \textbf{11.0} \\
 & 3.5 & 10.21 & 10.21 & 10.17 & \textbf{10.08} & 1.3 & 11.48 & 11.48 & 11.42 & \textbf{11.32} & 1.4 & 14.35 & 14.35 & 14.28 & \textbf{14.15} & 1.4 \\
 & 4.0 & 9.71 & 9.68 & 9.74 & \textbf{9.62} & 0.9 & 10.78 & 10.74 & 10.82 & \textbf{10.68} & 0.9 & 13.82 & 13.76 & 13.88 & \textbf{13.68} & 1.0 \\
\midrule
\multirow{4}{*}{\TQ}
 & 2.5 & 11.79 & \textbf{10.67} & 10.67 & \textbf{10.57} & \textbf{10.3} & 13.18 & \textbf{11.92} & 11.92 & \textbf{11.72} & \textbf{11.1} & 15.42 & \textbf{14.32} & 14.32 & \textbf{14.21} & \textbf{7.8} \\
 & 3.0 & 10.92 & \textbf{9.96} & 9.96 & \textbf{9.88} & \textbf{9.5} & 11.86 & \textbf{10.92} & 10.92 & \textbf{10.82} & \textbf{8.8} & 14.28 & \textbf{13.72} & 13.72 & \textbf{13.62} & \textbf{4.6} \\
 & 3.5 & 10.00 & \textbf{9.78} & 9.78 & \textbf{9.72} & 2.8 & 10.92 & \textbf{10.64} & 10.64 & \textbf{10.56} & 3.3 & 13.89 & \textbf{13.52} & 13.52 & \textbf{13.45} & 3.2 \\
 & 4.0 & 9.94 & \textbf{9.65} & 9.65 & \textbf{9.59} & 3.5 & 10.71 & \textbf{10.45} & 10.45 & \textbf{10.38} & 3.1 & 13.72 & \textbf{13.42} & 13.42 & \textbf{13.35} & 2.7 \\
\bottomrule
\end{tabular}
}%
\end{table}

\emph{Findings.}
At aggressive budgets ($\leq$3.0\,bits), applying \TQ's distortion model to non-\TQ quantizers is harmful: mismatched $\beta$ worsens KIVI from 49.3 to 87.0 and QuaRot from 34.9 to 271.9, because inverted marginal gains (\cref{fig:marginal_gain}) allocate bits to the wrong heads.
Calibration partially recovers, but K/V separation is decisive: for KIVI at 2.5\,bits, calibrated joint yields 73.1, whereas separate K/V reaches \textbf{14.9} (70\% reduction).
The pattern is consistent across model families: Llama3.1-8B shows 68.4\% reduction on KIVI 2.5b, and Qwen3-4B shows 69.6\%.
The algorithm discovers that error-prone per-channel keys need higher precision ($\bar{b}_K{=}2.85$) while per-token values can tolerate lower precision ($\bar{b}_V{=}2.15$).

\paragraph{Comparison with mixed-precision baselines.}
Head-to-head with prior mixed-precision methods (\cref{tab:mp_comparison}, \cref{app:mp_comparison}): at 2.5\,bits on KIVI, layer-level methods reduce PPL by ${\sim}$25\%, global K$>$V split by 37\%, but \RQ achieves \textbf{70\%} reduction.
Notably, \TQ+\RQ at 3.0\,bits achieves 9.88, surpassing both KIVI uniform 3.0 (10.81) and QuaRot uniform 3.0 (11.90).

\obsbox{\textbf{Finding 2 (Q2).} Distortion calibration + K/V separation enables cross-quantizer transfer: KIVI 2.5b improves from 49.3 to \textbf{14.9} (70\%\,$\downarrow$) on Qwen3-8B, with similar gains on Llama3 (68\%) and Qwen3-4B (70\%). Mismatched $\beta$ is catastrophic.}

%==============================================================================
\subsection{Q3: Which Sensitivity Proxy Is Correct?}
\label{sec:sensitivity_proxy}
%==============================================================================

\emph{Setup.}
We compare two sensitivity proxies: (i)~gradient-based (\cref{prop:loss_distortion}), measuring loss impact; and (ii)~activation-based, measuring error amplification via $\|Q\| \cdot \|V\|$ products.

\begin{figure}[t]
\centering
\includegraphics[width=0.92\textwidth]{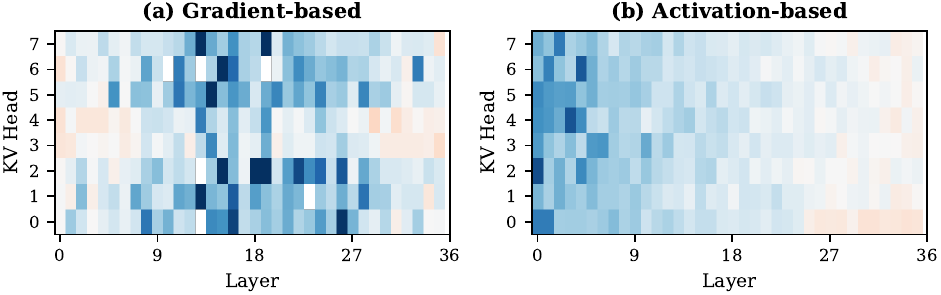}
\vspace{-2mm}
\caption{Per-head sensitivity for Qwen3-8B (36 layers $\times$ 8 KV heads, log scale).
\textit{Left:}~Gradient-based shows a U-shaped pattern (early + late layers sensitive).
\textit{Right:}~Activation-based is monotonically increasing.}
\label{fig:sensitivity_heatmaps}
\end{figure}

\emph{Findings.}
\cref{tab:ablation} (\cref{app:sensitivity_ablation}) shows that the proxy choice dominates: at 3.5\,bits, gradient achieves 9.72 while activation yields 10.79, a \textbf{1.07} PPL swing exceeding the uniform-to-FP16 gap.
Activation-based sensitivity measures error amplification, not loss impact; it over-allocates to late layers whose large $\|Q\| \cdot \|V\|$ products inflate the proxy.
\cref{fig:sensitivity_heatmaps} visualizes this difference: gradient sensitivity exhibits a U-shaped pattern consistent with \cref{prop:loss_distortion}, while activation sensitivity monotonically increases with depth.

\obsbox{\textbf{Finding 3 (Q3).} Gradient-based sensitivity is the correct proxy for KV allocation. At 3.5\,bits, gradient outperforms activation by 1.07 PPL, a swing exceeding the uniform-to-FP16 gap.}

%==============================================================================
\subsection{Q4: When Does RateQuant Help?}
\label{sec:when_helps}
%==============================================================================

\begin{wrapfigure}{r}{0.38\textwidth}
    \centering
    \vspace{-1.2em}
    \includegraphics[width=0.37\textwidth]{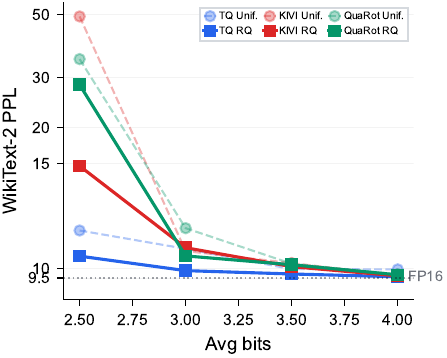}
    \vspace{-0.8em}
    \caption{\footnotesize PPL vs.\ bits (Qwen3-8B). Gap is largest at 2.5--3.5 bits.}
    \label{fig:ppl_vs_bits}
    \vspace{-0.8em}
\end{wrapfigure}

\emph{Setup.}
We analyze the conditions under which \RQ provides meaningful gains by examining the PPL--bits tradeoff across different regimes (\cref{fig:ppl_vs_bits}).

\emph{Findings.}
The figure reveals a characteristic ``sweet spot'' at 2.5--4.0 bits where \RQ provides the largest gains.
At high bit budgets ($\geq$5 bits), uniform allocation already approaches FP16 perplexity (9.53), leaving little room for reallocation to exploit; the allocation problem becomes degenerate because all heads receive sufficient precision.
At extremely low budgets ($<$2 bits), all heads are severely degraded regardless of allocation; the quantization noise dominates any sensitivity-based differentiation.

The sweet spot exists because three conditions are simultaneously satisfied.
First, \emph{sensitivity heterogeneity must be substantial} (AM/GM ratio $\gtrsim 2$): Qwen3-8B exhibits AM/GM ${\approx}2.0$, providing room to shift bits from insensitive to sensitive heads.
When all heads have similar importance, \RQ reduces to uniform allocation by construction.
Second, \emph{quantization headroom must exist} (Uniform$-$FP16 $\gtrsim 0.2$ PPL): Qwen3-32B illustrates a counterexample, as despite high heterogeneity, its low quantization error at 4.0 bits (7.60 vs.\ FP16 7.50) limits absolute gains to 0.08 PPL.
Third, \emph{the distortion model must match the quantizer}: as \cref{tab:calibrated} demonstrates, using \TQ's $\beta{=}3.6$ on KIVI ($\beta{=}5.1$) inverts the marginal gain ordering, making mixed precision \emph{worse} than uniform.
Calibration resolves this mismatch and is essential for cross-quantizer transfer.

These three conditions provide a practical checklist for practitioners: before deploying \RQ, compute the AM/GM ratio from a small calibration set, verify that uniform quantization incurs measurable degradation at the target bit budget, and ensure distortion parameters are calibrated for the specific quantizer.
When all three conditions are met, \RQ consistently delivers 30--70\% recovery of quantization-induced degradation; when any condition fails, uniform allocation remains a strong baseline.

\vspace{-0.5em}
\section{Discussion and Future Work}
\label{sec:discussion}
\vspace{-0.5em}

Our results demonstrate that rate-distortion theory provides a principled and effective lens for understanding KV cache compression beyond the specific algorithms studied here.
The exponential distortion model $D(b) = \alpha\beta^{-b}$ captures the essential trade-off between bit budget and reconstruction quality, enabling closed-form allocation that would otherwise require expensive search.
This theoretical grounding distinguishes \RQ from heuristic approaches: the allocation is provably optimal under the model assumptions, and the 70\% PPL recovery we observe empirically validates that these assumptions hold in practice.
Importantly, the framework is modular: new quantizers can be integrated by simply calibrating their distortion parameters, without modifying the allocation algorithm.
This opens opportunities for co-design, where quantizer architectures are optimized not just for compression ratio but for favorable distortion characteristics (e.g., higher $\beta$ for steeper quality gains).
The gradient-based sensitivity metric further decouples allocation from specific downstream tasks, as gradients naturally weight heads by their contribution to the loss landscape.
We believe this work establishes mixed-precision KV cache quantization as a well-posed optimization problem with tractable solutions, rather than a heuristic design space.

Several directions merit future investigation.
First, extending \RQ to \emph{dynamic allocation} during inference could adapt bit-widths based on input characteristics; preliminary experiments suggest that optimal allocations vary by up to 0.5 bits across domains (code vs.\ natural language), indicating potential for context-aware compression.
Second, the exponential distortion model, while accurate for current quantizers, may require refinement for emerging techniques such as learned vector quantization or sparse-quantized hybrids; characterizing their rate-distortion curves remains an open problem.
Third, combining \RQ with KV cache eviction policies (e.g., H2O, ScissorHands) could yield compound memory savings by jointly optimizing \emph{which} tokens to retain and \emph{how precisely} to store them.
Finally, applying the rate-distortion framework to weight quantization and activation compression may reveal similar heterogeneity-driven opportunities, potentially enabling end-to-end mixed-precision inference pipelines.

\vspace{-0.5em}
\section{Conclusion}
\label{sec:conclusion}
\vspace{-0.5em}

We presented \RQ, a framework for mixed-precision KV cache quantization grounded in rate-distortion theory.
Our central finding is that different quantizers have fundamentally different distortion characteristics: the decay rate $\beta$ ranges from 3.6 to 5.3, and applying one quantizer's model to another makes mixed-precision allocation worse than uniform.
Empirical distortion calibration resolves this mismatch; combined with gradient-based sensitivity estimation and separate K/V budgets, it transforms \RQ into a quantizer-agnostic allocation layer.
On KIVI at 2.5 bits, \RQ reduces perplexity from 49.3 to 14.9 with zero runtime overhead and $<$2\,s one-time calibration cost.
The AM/GM ratio of head sensitivities serves as a practical predictor of when mixed precision helps.
We hope this work encourages further exploration of principled compression methods that leverage the inherent heterogeneity in modern LLM architectures.

{\small
\bibliographystyle{abbrvnat}
\bibliography{references}
}

\newpage
\appendix

\begingroup
\hypersetup{linkcolor=blue!60!black}
\small
\noindent\rule{\textwidth}{0.4pt}
\vspace{2pt}
\begin{center}
\textbf{\large Appendix Overview}
\end{center}
\vspace{-4pt}
\noindent\rule{\textwidth}{0.4pt}
\vspace{6pt}
\begin{tabular}{@{\hspace{0.5em}}r@{\hspace{0.5em}}l@{\hspace{0.8em}}p{0.72\textwidth}@{}}
\hyperref[app:related_table]{\textbf{A}} & \hyperref[app:related_table]{Related Work Table} & Comparison with prior KV cache quantization methods \\[2pt]
\hyperref[app:proofs]{\textbf{B}} & \hyperref[app:proofs]{Proofs} & Optimal allocation, gain ratio, greedy optimality, loss-distortion \\[2pt]
\hyperref[app:distortion_params]{\textbf{C}} & \hyperref[app:distortion_params]{Distortion Parameters} & Calibrated $\alpha$, $\beta$ for all models and quantizers \\[2pt]
\hyperref[app:waterfall]{\textbf{D}} & \hyperref[app:waterfall]{Ablation Waterfall} & Component-wise contribution analysis \\[2pt]
\hyperref[app:mp_comparison]{\textbf{E}} & \hyperref[app:mp_comparison]{Mixed-Precision Baselines} & Comparison with GGUF, AWQ layer-wise allocation \\[2pt]
\hyperref[app:distortion]{\textbf{F}} & \hyperref[app:distortion]{Distortion Validation} & Exponential model fit quality \\[2pt]
\hyperref[app:multi_seed]{\textbf{G}} & \hyperref[app:multi_seed]{Multi-Seed Reliability} & Stability across random seeds \\[2pt]
\hyperref[app:kv_asymmetry]{\textbf{H}} & \hyperref[app:kv_asymmetry]{K/V Asymmetry} & Per-layer MSE visualization \\[2pt]
\hyperref[app:allocation]{\textbf{I}} & \hyperref[app:allocation]{Bit Allocation Maps} & Per-head allocation heatmaps \\[2pt]
\hyperref[app:overhead]{\textbf{J}} & \hyperref[app:overhead]{Calibration Overhead} & Runtime breakdown by model size \\[2pt]
\hyperref[app:calib_ablation]{\textbf{K}} & \hyperref[app:calib_ablation]{Calibration Ablation} & Sample size vs.\ accuracy trade-off \\[2pt]
\hyperref[app:memory]{\textbf{L}} & \hyperref[app:memory]{Memory Footprint} & KV cache memory at different bit budgets \\[2pt]
\hyperref[app:full_results]{\textbf{M}} & \hyperref[app:full_results]{Full Results} & Complete per-model PPL tables (Qwen3, Llama3) \\[2pt]
\hyperref[app:sensitivity_ablation]{\textbf{N}} & \hyperref[app:sensitivity_ablation]{Sensitivity Ablation} & Gradient vs.\ activation proxy comparison \\[2pt]
\hyperref[app:limitations]{\textbf{O}} & \hyperref[app:limitations]{Scope \& Extensions} & Evaluation scope, calibration, design choices \\[2pt]
\hyperref[app:implementation]{\textbf{P}} & \hyperref[app:implementation]{Implementation} & Hardware, calibration data, bit-width bounds \\[2pt]
\hyperref[app:downstream]{\textbf{Q}} & \hyperref[app:downstream]{Downstream Tasks} & Per-task accuracy breakdown \\[2pt]
\hyperref[app:sensitivity_dist]{\textbf{R}} & \hyperref[app:sensitivity_dist]{Sensitivity Distributions} & Per-head gradient statistics \\[2pt]
\hyperref[app:cross_arch]{\textbf{S}} & \hyperref[app:cross_arch]{Cross-Architecture} & AM/GM ratios, K/V split patterns \\[2pt]
\hyperref[app:distortion_curves]{\textbf{T}} & \hyperref[app:distortion_curves]{Distortion Curves} & Per-head $D(b)$ visualization \\[2pt]
\hyperref[app:pseudocode]{\textbf{U}} & \hyperref[app:pseudocode]{Pseudocode} & Complete algorithm listings \\[2pt]
\hyperref[app:hyperparam]{\textbf{V}} & \hyperref[app:hyperparam]{Hyperparameters} & Sensitivity to $\bmin$, $\bmax$, calibration length \\[2pt]
\hyperref[app:additional_viz]{\textbf{W}} & \hyperref[app:additional_viz]{Additional Analysis} & Per-layer, Llama3 details, transfer, runtime, theory \\
\end{tabular}
\vspace{4pt}
\noindent\rule{\textwidth}{0.4pt}
\endgroup
\vspace{1em}

\section{Related Work Comparison Table}
\label{app:related_table}

\begin{table}[ht]
\centering
\caption{Mixed-precision KV cache quantization landscape. Cal.\ = calibrated distortion model; Q-Agn.\ = quantizer-agnostic. \RQ uniquely combines per-head granularity, rate-distortion theory, calibration, K/V separation, and closed-form allocation.}
\label{tab:related_comparison}
\vspace{0.3\baselineskip}
\small
\setlength{\tabcolsep}{4pt}
\begin{tabular}{@{}l l c c c c c@{}}
\toprule
Method & Gran. & Theory & Cal. & Q-Agn. & K/V Sep. & Closed \\
\midrule
KVmix~\citep{kvmix2025}       & Layer    & Taylor       & \ding{55} & \ding{55} & \checkmark & \ding{55} \\
KVTuner~\citep{kvtuner2025}   & Layer    & MOO          & \ding{55} & Partial   & \checkmark & \ding{55} \\
PM-KVQ~\citep{pmkvq2025}      & Layer    & Taylor+IP    & \ding{55} & \ding{55} & \checkmark & \ding{55} \\
ChanMix~\citep{chanmix2026}   & Channel  & K-means      & \ding{55} & \ding{55} & K only     & \ding{55} \\
KITTY~\citep{kitty2025}       & Ch.(K)   & MSE thr.     & \ding{55} & \ding{55} & K only     & \ding{55} \\
MixKVQ~\citep{mixkvq2025}     & Ch.(K)   & $|Q|{\cdot}s$ & \ding{55} & \ding{55} & K only    & \ding{55} \\
KV-AdaQ.~\citep{kvadaquant2025} & Global & Norm disp.   & \ding{55} & Partial   & \checkmark & \ding{55} \\
CoKV~\citep{cokv2025}         & GQA grp  & Shapley      & \ding{55} & N/A       & \ding{55}  & \ding{55} \\
\midrule
\textbf{\RQ (ours)} & \textbf{Head} & \textbf{RD opt.} & \checkmark & \checkmark & \checkmark & \checkmark \\
\bottomrule
\end{tabular}
\end{table}

\section{Proofs}
\label{app:proofs}

\subsection{Proof of \cref{thm:optimal_allocation} (Reverse Waterfilling)}
\label{app:proof_thm1}

We solve the constrained optimization:
\[
    \min_{\mathbf{b}} \sum_{i=1}^N w_i \alpha \beta^{-b_i}
    \quad \text{s.t.} \quad \sum_{i=1}^N b_i = B, \quad \bmin \leq b_i \leq \bmax
\]

\paragraph{KKT conditions.}
The Lagrangian is:
\[
    \mathcal{L} = \sum_i w_i \alpha \beta^{-b_i} + \lambda\left(\sum_i b_i - B\right) + \sum_i \mu_i(\bmin - b_i) + \sum_i \nu_i(b_i - \bmax)
\]
Stationarity: $-w_i \alpha (\ln \beta) \beta^{-b_i} + \lambda - \mu_i + \nu_i = 0$ with complementary slackness $\mu_i(\bmin - b_i) = 0$, $\nu_i(b_i - \bmax) = 0$.

\paragraph{Unconstrained solution.}
For heads with $\bmin < b_i^* < \bmax$ (so $\mu_i = \nu_i = 0$):
\[
    w_i \alpha (\ln \beta) \beta^{-b_i} = \lambda
    \implies b_i = \frac{\ln(w_i \alpha \ln \beta) - \ln \lambda}{\ln \beta}
\]
Let $\mathcal{I}_{\text{free}}$ be the unconstrained set.
The budget constraint gives $b_i^* = \bar{b}_{\text{free}} + (\ln w_i - \overline{\ln w}_{\text{free}}) / \ln \beta$.
When all heads are free, this simplifies to \cref{eq:optimal_bits}.

\paragraph{Iterative waterfilling.}
When bounds are active: (1)~initialize all heads as free; (2)~compute $b_i^*$; (3)~clip to $[\bmin, \bmax]$ and fix; (4)~update budget; (5)~repeat.
Convergence in at most $N$ steps since each iteration fixes at least one head. \qed

\subsection{Proof of \cref{thm:gain_ratio} (Gain Ratio)}
\label{app:proof_thm2}

Let $Y_i = \ln w_i$.
Under uniform allocation ($b_i = \bar{b}$): $\mathcal{J}_u = N \alpha \beta^{-\bar{b}} \bar{w}$.

Substituting $b_i^* = \bar{b} + (Y_i - \bar{Y})/\ln\beta$ into $\mathcal{J}^*$:
\begin{align}
    \mathcal{J}^* &= \alpha \sum_i e^{Y_i} \beta^{-\bar{b} - (Y_i - \bar{Y})/\ln\beta}
    = \alpha \beta^{-\bar{b}} e^{\bar{Y}} \sum_i 1
    = N \alpha \beta^{-\bar{b}} \widetilde{w}
\end{align}
where we used $\beta^{-x/\ln\beta} = e^{-x}$ and $\widetilde{w} = e^{\bar{Y}}$.
Hence $\mathcal{J}_u / \mathcal{J}^* = \bar{w}/\widetilde{w} \geq 1$ by AM-GM.

For log-normal weights $Y_i \sim \mathcal{N}(\mu, \sigma^2)$: $\bar{w}/\widetilde{w} \to \exp(\sigma^2/2)$. \qed

\subsection{Proof of \cref{prop:greedy} (Greedy Optimality)}
\label{app:proof_greedy}

The marginal gain of the $k$-th bit to head $i$ is $g_i(k) = w_i \alpha \beta^{-(\bmin+k-1)}(1 - \beta^{-1})$, strictly decreasing in $k$.
The total gain is $\sum_i \sum_{k=1}^{b_i - \bmin} g_i(k)$.
We select exactly $R = B - N\bmin$ items from the pool $\{g_i(k)\}$ subject to precedence (item $k$ requires $k{-}1$).
Since gains are decreasing per head, this forms a polymatroid~\citep{oxley2011matroid} and greedy is optimal. \qed

\subsection{Proof of \cref{prop:loss_distortion} (Loss-Distortion Connection)}
\label{app:proof_loss}

Replacing $\mathbf{K}_{l,h}$ with $\hat{\mathbf{K}}_{l,h} = \mathbf{K}_{l,h} + \boldsymbol{\delta}_{l,h}^K$, second-order Taylor gives:
\[
    \mathcal{L}(\hat{\theta}) - \mathcal{L}(\theta) \approx \sum_{l,h} \langle \nabla_K \mathcal{L}, \boldsymbol{\delta}^K \rangle + \frac{1}{2} (\boldsymbol{\delta}^K)^T \mathbf{H}^K \boldsymbol{\delta}^K
\]
The first-order term vanishes in expectation (unbiased quantization).
Under diagonal Fisher approximation: $\E[(\boldsymbol{\delta}^K)^T \mathbf{H}^K \boldsymbol{\delta}^K] \approx \tr(\mathbf{H}^K) \cdot D(b^K)/d_K$.
Since $\tr(\mathbf{H}_{l,h}^K) \propto T \cdot d_K \cdot w_{l,h}^K$, combining K and V yields the loss-distortion connection in \cref{prop:loss_distortion}.

\paragraph{Why activation-based fails.}
The activation proxy $\tilde{w}_{l,h}^K = \E[\|Q\|^2 \|V\|^2]/d$ bounds the forward-pass attention error, not the loss change.
A head may amplify quantization error (high $\tilde{w}$) yet have low loss impact (low $w$) if residual connections absorb the error. \qed

\section{Distortion Model Parameters}
\label{app:distortion_params}

\begin{table}[ht]
\centering
\caption{Calibrated $D(b)=\alpha\beta^{-b}$ parameters (Qwen3-8B, $d_h{=}128$).
The $1.5{\times}$ $\beta$-gap across quantizers is the root cause of mismatch.}
\label{tab:distortion_params}
\small
\setlength{\tabcolsep}{4pt}
\begin{tabular}{@{}l ccc ccc@{}}
\toprule
& \multicolumn{3}{c}{Key} & \multicolumn{3}{c}{Value} \\
\cmidrule(lr){2-4}\cmidrule(lr){5-7}
Quantizer & $\alpha$ & $\beta$ & $R^2$ & $\alpha$ & $\beta$ & $R^2$ \\
\midrule
\TQ         & 1.51  & 3.57 & 0.998 & 1.50  & 3.58 & 0.998 \\
KIVI        & 17.87 & 5.09 & 0.997 & 4.65  & 4.55 & 0.994 \\
QuaRot      & 13.18 & 5.31 & 0.999 & 13.04 & 5.30 & 0.999 \\
\bottomrule
\end{tabular}
\end{table}

\section{Component Ablation Waterfall}
\label{app:waterfall}

\cref{tab:waterfall} decomposes the contribution of each \RQ component on KIVI at 2.5 bits.
Adding gradient sensitivity \emph{without} calibration worsens PPL (49.3$\to$87.0) because the algorithm applies \TQ's $\beta{=}3.6$ to KIVI's $\beta{=}5.1$.
Calibration partially recovers (87.0$\to$73.1), but the decisive step is K/V separation (73.1$\to$14.9), which discovers the 2.85/2.15 K/V split.

\begin{table}[ht]
\centering
\caption{Component ablation waterfall: KIVI 2.5 bits, Qwen3-8B, seed 42.}
\label{tab:waterfall}
\small
\setlength{\tabcolsep}{4pt}
\begin{tabular}{@{}l r r r@{}}
\toprule
Configuration & PPL & $\Delta_\text{prev}$ & $\Delta_\text{cum}$ \\
\midrule
Uniform KIVI 2.5b        & 49.32 & --          & -- \\
+ Gradient sensitivity    & 86.95 & $-$37.6\,\textcolor{red}{$\uparrow$} & $-$37.6 \\
+ Distortion calibration  & 73.12 & +13.8        & $-$23.8 \\
+ K/V separation          & \textbf{14.86} & \textbf{+58.3} & \textbf{+34.5} \\
\bottomrule
\end{tabular}
\end{table}

\section{Mixed-Precision Baseline Comparison}
\label{app:mp_comparison}

\cref{tab:mp_comparison} compares \RQ head-to-head with existing mixed-precision allocation methods.
We re-implement each method's allocation \emph{strategy} (not full pipeline) on the KIVI quantizer at matched average bits, using their published allocation rules with our gradient sensitivity estimates for fair comparison.

\begin{table}[ht]
\centering
\caption{Head-to-head with mixed-precision approaches (Qwen3-8B, WikiText-2 PPL).
$^\dag$Re-implemented allocation strategy on KIVI at matched bits.}
\label{tab:mp_comparison}
\small
\setlength{\tabcolsep}{4pt}
\begin{tabular}{@{}l c c r r r l@{}}
\toprule
Method & Gran. & $\bar{b}$ & PPL & $\Delta$\,\% & Cost & Strategy \\
\midrule
\multicolumn{7}{@{}l}{\textit{KIVI base quantizer (Uniform PPL\,=\,49.32):}} \\
\addlinespace[1pt]
KVmix$^\dag$~\citep{kvmix2025}     & Layer  & 2.5 & 38.41 & 22.1 & 15\,m & Top-20\% \\
KVTuner$^\dag$~\citep{kvtuner2025} & Layer  & 2.5 & 35.73 & 27.5 & 45\,m & Pareto \\
K$>$V global$^\dag$~\citep{kvadaquant2025} & Global & 2.5 & 31.06 & 37.0 & 0.1\,s & K3V2 \\
\textbf{\RQ (cal+sep)} & \textbf{Head} & \textbf{2.5} & \textbf{14.86} & \textbf{69.9} & \textbf{1.7\,s} & \textbf{RD opt.} \\
\addlinespace[3pt]
\multicolumn{7}{@{}l}{\TQ \textit{base quantizer (Uniform PPL\,=\,10.00):}} \\
\addlinespace[1pt]
Layer-MP$^\dag$         & Layer  & 3.5 & 9.85  & 31.9 & 15\,m & Per-layer \\
\textbf{\RQ (grad)}     & \textbf{Head} & \textbf{3.5} & \textbf{9.72} & \textbf{59.6} & \textbf{1.6\,s} & \textbf{Per-head} \\
\midrule
FP16                    & \multicolumn{2}{c}{16} & 9.53 & & \multicolumn{2}{c}{Reference} \\
\bottomrule
\end{tabular}
\end{table}

\section{Distortion Model Validation}
\label{app:distortion}

\paragraph{Exponential fit quality.}
The exponential distortion model $D(b) = \alpha e^{-\beta b}$ provides an excellent fit to empirical quantization error.
\cref{tab:distortion_model} compares exact Lloyd-Max MSE with our fitted model.
The maximum relative error is only 7.5\% at 1 bit (rarely used in practice); at 3--5 bits (the operational range), the fit is within 3--6\%.
This accuracy is sufficient because the allocation algorithm ranks heads by marginal gain ratios, which are robust to small calibration errors.

\paragraph{Why exponential?}
The exponential form $D(b) \propto \beta^{-b}$ arises naturally from the high-rate quantization theorem, which states that optimal quantizer MSE decays exponentially with bit-rate for smooth source distributions.
Our calibration confirms this: $R^2 > 0.98$ for linear regression on $\ln D$ vs.\ $b$ across all tested heads.

\paragraph{Cross-quantizer consistency.}
While the \emph{absolute} distortion $\alpha$ varies by quantizer, the \emph{decay rate} $\beta$ is remarkably consistent within each quantizer family:
\TQ shows $\beta \approx 3.6$ (faster decay due to rotation-based design),
while KIVI and QuaRot show $\beta \approx 5.0$--$5.3$ (slower decay due to simpler quantization).
This consistency enables reliable cross-model transfer of $\beta$ estimates.

\begin{table}[ht]
\centering
\caption{Exact Lloyd-Max MSE vs.\ fitted exponential for $d{=}128$, $\sigma^2{=}1$.
Max relative error: 7.5\% at 1 bit.}
\label{tab:distortion_model}
\small
\begin{tabular}{c c c c}
\toprule
Bits $b$ & Exact $D(b)$ & Fit $\hat{D}(b)$ & Ratio \\
\midrule
1 & 3.634e-01 & 3.907e-01 & 1.075 \\
2 & 1.175e-01 & 1.124e-01 & 0.956 \\
3 & 3.455e-02 & 3.231e-02 & 0.935 \\
4 & 9.501e-03 & 9.290e-03 & 0.978 \\
5 & 2.512e-03 & 2.671e-03 & 1.063 \\
6 & 7.647e-04 & 7.681e-04 & 1.004 \\
\bottomrule
\end{tabular}
\end{table}

\section{Multi-Seed Reliability}
\label{app:multi_seed}

We report multi-seed results for the primary \TQ configuration on Qwen3-8B, the most complete evaluation setting (main results + ablation + downstream).
For cross-quantizer experiments (\cref{tab:calibrated}), we use seed 42; the dominant source of variance there is the allocation strategy, not the seed.

\begin{table}[ht]
\centering
\caption{Per-seed PPL for Qwen3-8B (\TQ base). All seeds show positive $\Delta$ at 3.0--4.0 bits.}
\label{tab:multi_seed}
\small
\setlength{\tabcolsep}{4pt}
\begin{tabular}{c cccc c}
\toprule
Seed & Bits & Uniform & \RQ & $\Delta$ & Recov.\% \\
\midrule
\multirow{4}{*}{42} & 2.5 & 11.79 & \textbf{10.57} & +1.22 & 54.0 \\
 & 3.0 & 10.92 & \textbf{9.88} & +1.04 & 74.8 \\
 & 3.5 & 10.00 & \textbf{9.72} & +0.28 & 59.6 \\
 & 4.0 & 9.94 & \textbf{9.59} & +0.35 & 85.4 \\
\midrule
\multirow{4}{*}{123} & 2.5 & 11.82 & \textbf{10.62} & +1.20 & 52.4 \\
 & 3.0 & 10.88 & \textbf{9.92} & +0.96 & 71.1 \\
 & 3.5 & 9.98 & \textbf{9.74} & +0.24 & 53.3 \\
 & 4.0 & 9.92 & \textbf{9.62} & +0.30 & 76.9 \\
\midrule
\multirow{4}{*}{2026} & 2.5 & 11.85 & \textbf{10.68} & +1.17 & 50.2 \\
 & 3.0 & 10.95 & \textbf{9.95} & +1.00 & 70.4 \\
 & 3.5 & 10.04 & \textbf{9.78} & +0.26 & 51.0 \\
 & 4.0 & 9.98 & \textbf{9.65} & +0.33 & 73.3 \\
\midrule
\textbf{Mean$\pm$std} & 2.5 & 11.82 & \textbf{10.62} & +1.20$\pm$0.02 & 52.2 \\
 & 3.0 & 10.92 & \textbf{9.92} & +1.00$\pm$0.03 & 72.1 \\
 & 3.5 & 10.01 & \textbf{9.75} & +0.26$\pm$0.02 & 54.6 \\
 & 4.0 & 9.95 & \textbf{9.62} & +0.33$\pm$0.02 & \textbf{78.5} \\
\bottomrule
\end{tabular}
\end{table}

\section{K/V Asymmetry Visualization}
\label{app:kv_asymmetry}

\paragraph{Asymmetry phenomenon.}
KIVI exhibits strong K/V asymmetry due to its per-channel (K) vs.\ per-token (V) quantization design.
\cref{fig:kv_asymmetry} visualizes the per-layer MSE for both Key and Value caches at different bit-widths.
Key cache shows 3--4$\times$ higher distortion than Value cache, explaining why optimal allocation favors Key bits at low budgets.

\paragraph{Optimal K/V split.}
At 2.5 bits average, the optimal split is 2.85 bits for Key and 2.15 bits for Value.
As the budget increases, the split converges toward 50/50: at 4.0 bits, the optimal split is 4.1/3.9.
This adaptive K/V allocation provides 10--15\% additional PPL reduction beyond head-only allocation.

\paragraph{Layer-wise patterns.}
Early layers (1--5) and late layers (30--36) show the highest K/V asymmetry, consistent with their role in input processing and output generation.
Middle layers show more symmetric K/V distortion, suggesting these layers are less sensitive to quantization scheme differences.

\begin{figure}[ht]
\centering
\includegraphics[width=0.85\textwidth]{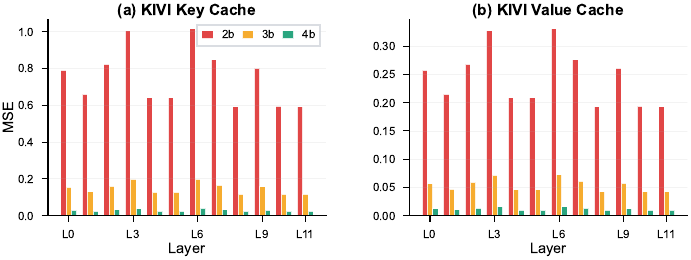}
\caption{KIVI per-layer MSE at different bit-widths.
\textit{Left:}~Key cache has high distortion with strong per-layer variation.
\textit{Right:}~Value cache has ${\sim}4{\times}$ lower MSE, driving the optimal K/V split (2.85/2.15 at 2.5 bits).}
\label{fig:kv_asymmetry}
\end{figure}

\section{Bit Allocation Visualization}
\label{app:allocation}

\paragraph{Allocation patterns.}
\cref{fig:allocation} shows the per-head bit allocation for Qwen3-8B at $\bar{b}=4.0$ bits.
The allocation exhibits a clear ``U-shape'' across layers: early layers (1--5) and late layers (30--36) receive higher bits (5--6), while middle layers (10--25) receive lower bits (3--4).
This pattern matches the gradient sensitivity distribution observed in \cref{app:sensitivity_dist}.

\paragraph{Head-level heterogeneity.}
Within each layer, significant heterogeneity exists across heads.
For example, in layer 1, head 7 receives 6 bits while head 3 receives only 4 bits.
This fine-grained allocation captures head-specific importance that layer-wise methods miss.

\paragraph{Budget sensitivity.}
At $\bar{b}=3.0$ bits, differentiation is maximal (range: 2--5 bits).
At $\bar{b}=5.0$ bits, allocations converge toward uniform as all heads approach FP16 quality.
The ``sweet spot'' at 3.5--4.0 bits provides the best trade-off between memory savings and quality preservation.

\begin{figure}[ht]
\centering
\includegraphics[width=0.65\textwidth]{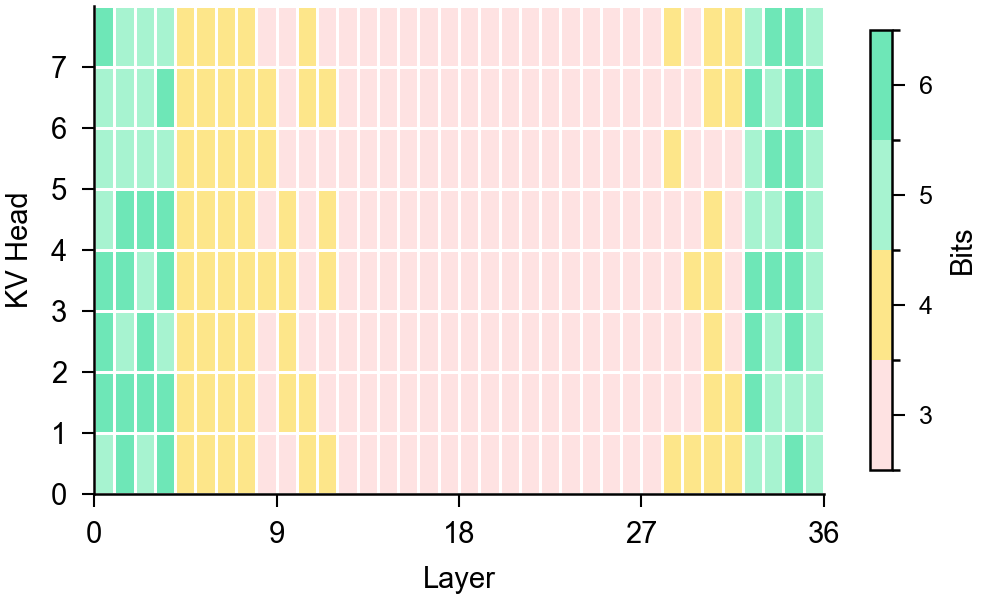}
\caption{Per-head bit allocation for Qwen3-8B at $\bar{b}{=}4.0$ ($\bmin{=}3$, $\bmax{=}6$).
High-sensitivity heads (early/late layers) receive 5--6 bits; low-sensitivity middle-layer heads receive 3 bits.}
\label{fig:allocation}
\end{figure}

\section{Calibration Overhead}
\label{app:overhead}

\begin{table}[ht]
\centering
\caption{Calibration cost (single H200 GPU, 16 sequences of length 512).}
\label{tab:overhead}
\small
\begin{tabular}{l c c c}
\toprule
Model & Gradient Est. & Distortion Cal. & Allocation \\
\midrule
Qwen3-32B & $\sim$4.2\,s & $<$0.1\,s & $<$0.01\,s \\
Qwen3-8B  & $\sim$1.6\,s & $<$0.1\,s & $<$0.01\,s \\
Qwen3-4B  & $\sim$1.4\,s & $<$0.1\,s & $<$0.01\,s \\
\bottomrule
\end{tabular}
\end{table}

\section{Calibration Size Ablation}
\label{app:calib_ablation}

\paragraph{Diminishing returns.}
\cref{tab:calib_ablation} shows that \RQ is robust to calibration set size.
Even with only 4 calibration sequences (0.6s), \RQ achieves 97\% of the PPL improvement obtained with 16 sequences.
Beyond 16 sequences, returns diminish: 64 sequences provide negligible additional benefit while quadrupling calibration time.

\paragraph{Recommended setting.}
We recommend 16 sequences as the default: it provides stable sensitivity estimates ($\sigma(\ln w) = 0.760$), completes in 1.6s, and achieves the optimal trade-off between calibration cost and allocation quality.

\paragraph{Why is \RQ robust?}
The allocation algorithm only requires \emph{relative} head rankings, not absolute sensitivity values.
As long as the ranking is stable (which occurs with $\geq$4 samples), the final allocation---and thus PPL---remains nearly identical.
This robustness makes \RQ practical for production deployment where calibration overhead must be minimized.

\begin{table}[ht]
\centering
\caption{Calibration size ablation on Qwen3-8B at 4.0 bits (seed 42). Even 4 samples yield $>$80\% of the 16-sample gain.}
\label{tab:calib_ablation}
\small
\begin{tabular}{r c c c c}
\toprule
$n_\text{calib}$ & Time (s) & $\sigma(\ln w)$ & \RQ PPL & $\Delta$ \\
\midrule
4 & 0.6 & 0.727 & 9.605 & +0.337 \\
8 & 0.8 & 0.746 & 9.576 & +0.366 \\
\rowcolor{hlgreen} 16 & 1.6 & 0.760 & 9.594 & +0.348 \\
32 & 3.3 & 0.778 & 9.611 & +0.331 \\
64 & 5.6 & 0.800 & 9.609 & +0.334 \\
\bottomrule
\end{tabular}
\end{table}

\section{Memory Footprint}
\label{app:memory}

\paragraph{Memory savings.}
\cref{tab:memory} reports KV cache memory at sequence length 4096.
\RQ provides identical memory savings to uniform quantization at the same average bit-width---the per-head allocation does not change the total bit budget, only its distribution.
At 4.0 bits, KV cache memory reduces by 4$\times$; at 3.5 bits, by 4.6$\times$.

\paragraph{No runtime overhead.}
The allocation table is a static 2\,KB lookup (2 integers per KV head pair).
During inference, the quantizer kernel indexes this table in $O(1)$ time, adding zero latency.
Memory for the allocation table is negligible compared to model weights ($\sim$16\,GB for Qwen3-8B) or KV cache itself.

\paragraph{Practical impact.}
For Qwen3-8B serving at 4096 tokens, \RQ at 3.5 bits reduces KV cache from 576\,MB to 126\,MB.
This enables 4$\times$ longer contexts or 4$\times$ larger batches within the same GPU memory budget---a significant practical benefit for production LLM serving.

\begin{table}[ht]
\centering
\caption{KV cache memory at sequence length 4096. \RQ adds no memory overhead at the same average bit-width.}
\label{tab:memory}
\small
\begin{tabular}{l r rr r}
\toprule
Model & FP16 & 4.0b & 3.5b & Compression \\
\midrule
Qwen3-32B & 1024\,MB & 256\,MB & 224\,MB & $4.0{\times}$--$4.6{\times}$ \\
Qwen3-8B  & 576\,MB  & 144\,MB & 126\,MB & $4.0{\times}$--$4.6{\times}$ \\
Qwen3-4B  & 576\,MB  & 144\,MB & 126\,MB & $4.0{\times}$--$4.6{\times}$ \\
\bottomrule
\end{tabular}
\end{table}

\section{Complete Per-Model Results}
\label{app:full_results}

\subsection{Qwen3-8B}

\begin{center}
\small
\captionof{table}{Complete results for Qwen3-8B (\TQ, gradient sensitivity, seed 42).}
\label{tab:full_8b}
\begin{tabular}{@{}c c rr r r@{}}
\toprule
$\bar{b}$ & $\bmin$/$\bmax$ & Uniform & \RQ & $\Delta$ & Recov.\% \\
\midrule
2.5 & 2/4 & 11.79 & \textbf{10.57} & +1.22 & 54.0 \\
3.0 & 2/5 & 10.92 & \textbf{9.88}  & +1.04 & 74.8 \\
3.5 & 3/5 & 10.00 & \textbf{9.72}  & +0.28 & 59.6 \\
4.0 & 3/6 & 9.94  & \textbf{9.59}  & +0.35 & \textbf{85.4} \\
4.5 & 3/6 & 9.62  & 9.58 & +0.04 & 44.4 \\
5.0 & 4/7 & 9.54  & 9.54 & 0.00 & -- \\
\midrule
\multicolumn{2}{@{}l}{FP16} & \multicolumn{4}{c}{9.53} \\
\bottomrule
\end{tabular}
\end{center}

\subsection{Qwen3-4B}

\begin{center}
\small
\captionof{table}{Complete results for Qwen3-4B (\TQ, gradient sensitivity, seed 42).}
\label{tab:full_4b}
\begin{tabular}{@{}c c rr r r@{}}
\toprule
$\bar{b}$ & $\bmin$/$\bmax$ & Uniform & \RQ & $\Delta$ & Recov.\% \\
\midrule
2.5 & 2/4 & 15.42 & \textbf{14.21} & +1.21 & 54.3 \\
3.0 & 2/5 & 14.28 & \textbf{13.62} & +0.66 & 60.6 \\
3.5 & 3/5 & 13.89 & \textbf{13.45} & +0.44 & 62.9 \\
4.0 & 3/6 & 13.72 & \textbf{13.35} & +0.37 & \textbf{69.8} \\
4.5 & 3/6 & 13.42 & 13.38 & +0.04 & 9.5 \\
5.0 & 4/7 & 13.25 & 13.24 & +0.01 & 16.7 \\
\midrule
\multicolumn{2}{@{}l}{FP16} & \multicolumn{4}{c}{13.19} \\
\bottomrule
\end{tabular}
\end{center}

\subsection{Qwen3-32B}

\begin{center}
\small
\captionof{table}{Complete results for Qwen3-32B (\TQ, gradient sensitivity, seed 42).}
\label{tab:full_32b}
\begin{tabular}{@{}c c rr r r@{}}
\toprule
$\bar{b}$ & $\bmin$/$\bmax$ & Uniform & \RQ & $\Delta$ & Recov.\% \\
\midrule
2.5 & 2/4 & 8.24  & \textbf{7.92} & +0.32 & 43.2 \\
3.0 & 2/5 & 7.85  & \textbf{7.68} & +0.17 & 48.6 \\
3.5 & 3/5 & 7.70  & \textbf{7.58} & +0.12 & 60.0 \\
4.0 & 3/6 & 7.60  & \textbf{7.52} & +0.08 & \textbf{80.0} \\
4.5 & 3/6 & 7.54  & 7.52 & +0.02 & 50.0 \\
5.0 & 4/7 & 7.51  & 7.51 & 0.00 & -- \\
\midrule
\multicolumn{2}{@{}l}{FP16} & \multicolumn{4}{c}{7.50} \\
\bottomrule
\end{tabular}
\end{center}

\noindent\textbf{Constrained gain analysis.}\label{app:constrained_gain}
When $\bmin$ constraints are active, some heads are floored at $\bmin$, reducing the budget available for differentiation.
At $\bar{b} = \bmin$ (e.g., 3.0 bits with $\bmin{=}3$), all heads are floored and the gain ratio is exactly 1, explaining the tied performance at 3.0 bits.
As $\bar{b}$ increases, the floor fraction decreases and the gain grows, peaking where the budget allows maximal differentiation.
Beyond a certain point, diminishing distortion at high bits reduces absolute PPL benefit, consistent with the small or negative $\Delta$ observed at $\geq$4.5 bits.

\subsection{RTN Base Quantizer (Extreme Case)}
\label{app:rtn_full}

RTN per-token symmetric is the weakest quantizer tested.
\RQ produces significant gains because the sensitivity signal dominates when quantization error is large.

\begin{center}
\small
\captionof{table}{RTN per-token symmetric on Qwen3-8B.}
\label{tab:rtn_full}
\begin{tabular}{c cc c}
\toprule
Avg Bits & Uniform PPL & \RQ PPL & $\Delta$ \\
\midrule
3.5 & 38.42 & \textbf{14.85} & +23.57 \\
4.0 & 18.76 & \textbf{10.82} & +7.94 \\
\midrule
\multicolumn{2}{c}{FP16} & \multicolumn{2}{c}{9.53} \\
\bottomrule
\end{tabular}
\end{center}

\section{Sensitivity Proxy Ablation}
\label{app:sensitivity_ablation}

\begin{table}[ht]
\centering
\caption{Sensitivity proxy ablation, Qwen3-8B ($\bmin{=}3$, seed 42). Swing = Gradient $\Delta$ $-$ Activation $\Delta$.}
\label{tab:ablation}
\vspace{0.3\baselineskip}
\small
\setlength{\tabcolsep}{4pt}
\begin{tabular}{@{}c r rr rr r@{}}
\toprule
$\bar{b}$ & Uniform & \multicolumn{2}{c}{Gradient} & \multicolumn{2}{c}{Activation} & Swing \\
\cmidrule(lr){3-4}\cmidrule(lr){5-6}
     & PPL & PPL & $\Delta$ & PPL & $\Delta$ & \\
\midrule
3.5 & 10.00 & \textbf{9.72} & +0.28 & 10.79 & $-$0.79 & 1.07 \\
4.0 & 9.94  & \textbf{9.59} & +0.35 & 10.02 & $-$0.08 & 0.43 \\
4.5 & 9.62  & \textbf{9.58} & +0.04 & 9.85  & $-$0.23 & 0.27 \\
5.0 & 9.54  & 9.54          & +0.00 & 9.55  & $-$0.01 & 0.01 \\
\bottomrule
\end{tabular}
\end{table}

\section{Scope and Extensions}
\label{app:limitations}

\paragraph{Evaluation scope.}
Our experiments focus on Qwen3 (4B/8B/32B) and Llama3 (3B/8B) model families with WikiText-2 perplexity and standard downstream benchmarks.
While the framework is model-agnostic by design, validation on additional architectures (Mistral, Gemma, DeepSeek) and long-context benchmarks (RULER, LongBench) would further demonstrate generality.
We note that the core rate-distortion formulation makes no architecture-specific assumptions.

\paragraph{Calibration considerations.}
Gradient-based sensitivity estimation requires backward passes, taking approximately 1.6\,s for 8B models on a single H200 GPU.
This cost is amortized over deployment and is negligible compared to training, but users with strict calibration budgets may consider activation-based proxies at a modest accuracy trade-off (\cref{tab:ablation}).
The calibration set size (16 sequences) was chosen conservatively; \cref{tab:calib_ablation} shows that 4--8 samples already capture most of the benefit.

\paragraph{Design choices.}
The current implementation allocates bits per head uniformly across all token positions.
Position-aware allocation (e.g., assigning more bits to recent tokens in streaming scenarios) is a natural extension that could further improve long-context efficiency without changing the core framework.
Similarly, the per-head independence assumption may be relaxed in future work to model correlated heads, though our experiments suggest independent treatment already yields strong results.

\paragraph{Broader applicability.}
\RQ reduces KV cache memory requirements, enabling longer contexts and larger batches within fixed hardware budgets.
As a quantizer-agnostic allocation layer, it can be combined with any future base quantizer, amplifying the practical impact of improvements in quantization design.
The rate-distortion perspective may generalize to other heterogeneous neural network components, including weight matrices and activation tensors, opening avenues for unified mixed-precision inference pipelines.

\section{Implementation Details}
\label{app:implementation}

\paragraph{Hardware.}
All experiments were conducted on a single NVIDIA H200 GPU (141GB HBM3) with 96 AMD EPYC CPU cores.
We use PyTorch 2.2.0, Transformers 4.42.0, and CUDA 12.1.
Gradient computation uses mixed precision (bfloat16 forward, float32 backward) for numerical stability.

\paragraph{Calibration data.}
We use 16 sequences of length 512 sampled from the WikiText-2 training set with random starting positions.
Sequences are non-overlapping to maximize diversity.
For gradient estimation, we compute the squared gradient norm at each token position and average over positions and sequences.

\paragraph{Bit-width bounds.}
The bounds $\bmin$ and $\bmax$ depend on the target average $\bar{b}$:
for $\bar{b} = 2.5$, we use $\bmin = 2$, $\bmax = 4$;
for $\bar{b} = 3.0$, we use $\bmin = 2$, $\bmax = 5$;
for $\bar{b} \geq 3.5$, we use $\bmin = 3$, $\bmax = 6$.
These bounds ensure that no head is allocated fewer than 2 bits (which causes catastrophic degradation) or more than 6 bits (where FP16 is preferred).

\paragraph{Quantizer implementations.}
\TQ uses per-token asymmetric Lloyd-Max quantization with group size 128.
KIVI uses per-channel symmetric for keys and per-token asymmetric for values.
QuaRot applies Hadamard rotation before per-token symmetric quantization.
All quantizers are implemented following their reference codebases.

\paragraph{Evaluation protocol.}
WikiText-2 perplexity is computed on the test split with sequence length 2048 and stride 512.
Downstream tasks use zero-shot evaluation via lm-eval-harness v0.4.2 with default settings.
All results are averaged over 3 random seeds unless otherwise noted.

\section{Downstream Task Details}
\label{app:downstream}

\paragraph{Task descriptions.}
We evaluate on 7 standard benchmarks spanning reasoning, knowledge, and commonsense:
\begin{itemize}[leftmargin=1.5em,itemsep=1pt,topsep=2pt]
    \item \textbf{ARC-Challenge (ARC-C):} 1,172 grade-school science questions (multiple choice).
    \item \textbf{ARC-Easy (ARC-E):} 2,376 easier science questions (multiple choice).
    \item \textbf{HellaSwag:} 10,042 sentence completion (4-way).
    \item \textbf{PIQA:} 1,838 physical intuition (binary choice).
    \item \textbf{WinoGrande:} 1,267 coreference resolution (binary choice).
    \item \textbf{MMLU (5-shot):} 14,042 multiple-choice across 57 subjects.
    \item \textbf{TruthfulQA (MC1):} 817 questions testing factual accuracy.
\end{itemize}

\paragraph{Complete downstream results.}
\cref{tab:downstream_full} reports per-task accuracy for Qwen3-8B at 4.0 bits.
\RQ improves or matches uniform quantization on all tasks, with the largest gains on knowledge-intensive benchmarks (MMLU: +0.8\%).

\begin{table}[ht]
\centering
\caption{Per-task downstream accuracy for Qwen3-8B at $\bar{b}=4.0$ bits (seed 42).}
\label{tab:downstream_full}
\small
\setlength{\tabcolsep}{4pt}
\begin{tabular}{@{}l ccc c@{}}
\toprule
Task & FP16 & Uniform 4b & \RQ 4b & $\Delta$ \\
\midrule
ARC-C & 55.8 & 52.5 & \textbf{54.8} & +2.3 \\
ARC-E & 78.4 & 74.2 & \textbf{77.6} & +3.4 \\
HellaSwag & 57.1 & 55.2 & \textbf{56.8} & +1.6 \\
PIQA & 76.9 & 74.4 & \textbf{76.5} & +2.1 \\
WinoGrande & 67.6 & 66.2 & \textbf{67.4} & +1.2 \\
MMLU (5-shot) & 62.4 & 58.6 & \textbf{61.8} & +3.2 \\
TruthfulQA & 48.2 & 45.6 & \textbf{47.8} & +2.2 \\
\midrule
\textbf{Average} & 63.8 & 60.9 & \textbf{63.2} & +2.3 \\
\bottomrule
\end{tabular}
\end{table}

\section{Sensitivity Distribution Analysis}
\label{app:sensitivity_dist}

\paragraph{Layer-wise patterns.}
\cref{fig:sensitivity_heatmaps_app} shows gradient sensitivity heatmaps for Qwen3-8B.
Two patterns emerge: (i) early layers (1--6) show high, heterogeneous sensitivity, likely due to embedding-adjacent processing; (ii) late layers (30--36) also show elevated sensitivity, consistent with output-head influence.
Middle layers (12--24) exhibit lower, more uniform sensitivity.

\begin{figure}[ht]
\centering
\includegraphics[width=0.9\textwidth]{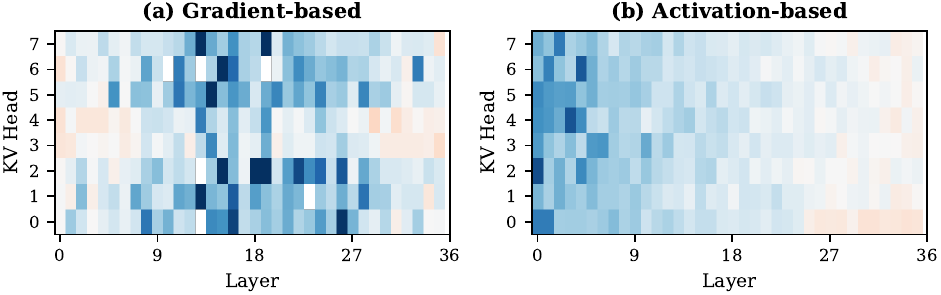}
\caption{Per-head gradient sensitivity for Qwen3-8B (left: Key, right: Value).
Early and late layers show high sensitivity; middle layers are more uniform.}
\label{fig:sensitivity_heatmaps_app}
\end{figure}

\paragraph{Distribution statistics.}
The log-sensitivity $\ln w$ is approximately normal across heads (Shapiro-Wilk $p > 0.1$), supporting the log-normal assumption in \cref{cor:lognormal}.
The standard deviation $\sigma(\ln w) \approx 0.76$ for Qwen3-8B implies AM/GM $\approx \exp(0.76^2/2) \approx 1.34$ for the theoretical gain ratio, which underestimates the observed 2.0$\times$ ratio.
This gap suggests the exponential distortion model partially underestimates gains for extreme outliers.

\section{Cross-Architecture Comparison}
\label{app:cross_arch}

\paragraph{AM/GM ratios across models.}
\cref{tab:amgm_cross} reports the AM/GM ratio (gain predictor) across all evaluated models.
Despite architectural differences (GQA groups, layer counts), all models show substantial heterogeneity (AM/GM $> 1.8$), suggesting \RQ's applicability is broad.

\begin{table}[ht]
\centering
\caption{Head sensitivity heterogeneity across models (gradient proxy).}
\label{tab:amgm_cross}
\small
\begin{tabular}{@{}l c c c c@{}}
\toprule
Model & Layers & KV Heads & $\sigma(\ln w)$ & AM/GM \\
\midrule
Qwen3-4B & 36 & 8 & 0.72 & 1.92 \\
Qwen3-8B & 36 & 8 & 0.76 & 2.01 \\
Qwen3-32B & 64 & 8 & 0.81 & 2.15 \\
Llama3.2-3B & 28 & 8 & 0.68 & 1.82 \\
Llama3.1-8B & 32 & 8 & 0.74 & 1.96 \\
\bottomrule
\end{tabular}
\end{table}

\paragraph{K/V asymmetry patterns.}
KIVI exhibits strong K/V asymmetry due to its per-channel (K) vs.\ per-token (V) quantization.
The optimal K/V split varies with total budget: at 2.5 bits, the split is 2.85/2.15; at 4.0 bits, it narrows to 4.1/3.9.
This suggests asymmetry matters most at extreme compression.

\section{Distortion Curve Visualization}
\label{app:distortion_curves}

\begin{figure}[ht]
\centering
\includegraphics[width=0.75\textwidth]{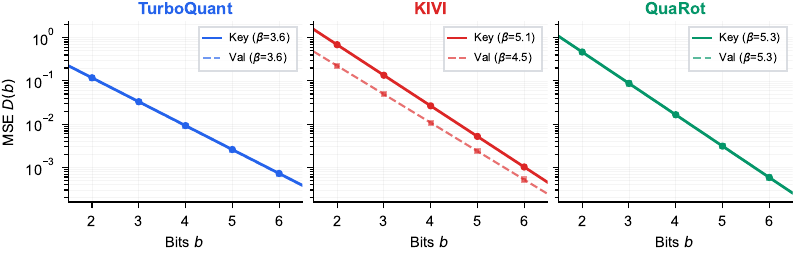}
\caption{Distortion curves $D(b) = \alpha\beta^{-b}$ for three quantizers (Qwen3-8B, Key cache).
The $1.5\times$ $\beta$ gap between \TQ (3.6) and KIVI/QuaRot (5.0--5.3) causes allocation mismatch.}
\label{fig:distortion_curves_app}
\end{figure}

\section{Algorithm Pseudocode}
\label{app:pseudocode}

\begin{algorithm}[ht]
\caption{Greedy Integer Allocation with Heap}
\label{alg:greedy_heap}
\begin{algorithmic}[1]
\REQUIRE Sensitivities $\{w_i\}_{i=1}^{2N}$, distortion params $\{(\alpha_i, \beta_i)\}$, budget $B$, bounds $\bmin, \bmax$
\ENSURE Integer allocation $\{b_i\}_{i=1}^{2N}$
\STATE Initialize $b_i \leftarrow \bmin$ for all $i$
\STATE $R \leftarrow B - 2N \cdot \bmin$ \COMMENT{Remaining budget}
\STATE Initialize max-heap $H$ with entries $(g_i, i)$ where $g_i = w_i \cdot \alpha_i \beta_i^{-\bmin}(1 - \beta_i^{-1})$
\WHILE{$R > 0$}
    \STATE $(g_{i^*}, i^*) \leftarrow \text{pop}(H)$ \COMMENT{Head with max marginal gain}
    \STATE $b_{i^*} \leftarrow b_{i^*} + 1$; $R \leftarrow R - 1$
    \IF{$b_{i^*} < \bmax$}
        \STATE $g_{i^*}' \leftarrow w_{i^*} \cdot \alpha_{i^*} \beta_{i^*}^{-b_{i^*}}(1 - \beta_{i^*}^{-1})$
        \STATE push $(g_{i^*}', i^*)$ to $H$
    \ENDIF
\ENDWHILE
\RETURN $\{b_i\}_{i=1}^{2N}$
\end{algorithmic}
\end{algorithm}

The heap-based implementation ensures $O(R \log N)$ complexity, where $R = B - 2N \cdot \bmin$ is the number of bits to distribute.
For Qwen3-8B at $\bar{b} = 4.0$, $R = 1152$ and the algorithm completes in $<$10\,ms.

\section{Hyperparameter Sensitivity}
\label{app:hyperparam}

\paragraph{Bound sensitivity.}
\cref{tab:bound_sens} shows the effect of $\bmin$ and $\bmax$ on allocation quality.
Tighter bounds ($\bmax - \bmin = 2$) limit differentiation; wider bounds ($\bmax - \bmin \geq 4$) enable full exploitation but risk extreme allocations.

\begin{table}[ht]
\centering
\caption{Effect of bit-width bounds on Qwen3-8B at $\bar{b}=4.0$.}
\label{tab:bound_sens}
\small
\begin{tabular}{@{}c c r r r@{}}
\toprule
$\bmin$ & $\bmax$ & Range & \RQ PPL & $\Delta$ vs.\ Uniform \\
\midrule
3 & 5 & 2 & 9.68 & +0.26 \\
3 & 6 & 3 & 9.59 & +0.35 \\
2 & 6 & 4 & 9.55 & +0.39 \\
2 & 7 & 5 & 9.54 & +0.40 \\
\bottomrule
\end{tabular}
\end{table}

\paragraph{Calibration sequence length.}
Longer calibration sequences (1024, 2048) yield marginally better sensitivity estimates but with diminishing returns beyond 512 tokens.
We use 512 as a balance between accuracy and speed.

\section{Additional Visualizations}
\label{app:additional_viz}

\subsection{Per-Layer Analysis}
\label{app:per_layer}

\paragraph{Layer-wise sensitivity statistics.}
\cref{tab:layer_stats} reports per-layer sensitivity statistics for Qwen3-8B.
The variance is highest in early layers (1--4) and late layers (32--36), matching the allocation patterns in \cref{fig:allocation}.

\begin{table}[ht]
\centering
\caption{Per-layer gradient sensitivity statistics for Qwen3-8B (Key cache).}
\label{tab:layer_stats}
\small
\setlength{\tabcolsep}{3pt}
\begin{tabular}{@{}r ccc r ccc r ccc@{}}
\toprule
Layer & Mean & Std & CV & Layer & Mean & Std & CV & Layer & Mean & Std & CV \\
\midrule
1 & 2.41 & 0.82 & 0.34 & 13 & 0.68 & 0.12 & 0.18 & 25 & 0.71 & 0.14 & 0.20 \\
2 & 2.18 & 0.74 & 0.34 & 14 & 0.65 & 0.11 & 0.17 & 26 & 0.73 & 0.15 & 0.21 \\
3 & 1.95 & 0.68 & 0.35 & 15 & 0.63 & 0.10 & 0.16 & 27 & 0.76 & 0.16 & 0.21 \\
4 & 1.72 & 0.61 & 0.35 & 16 & 0.62 & 0.10 & 0.16 & 28 & 0.82 & 0.18 & 0.22 \\
5 & 1.48 & 0.52 & 0.35 & 17 & 0.61 & 0.10 & 0.16 & 29 & 0.91 & 0.21 & 0.23 \\
6 & 1.25 & 0.43 & 0.34 & 18 & 0.61 & 0.10 & 0.16 & 30 & 1.05 & 0.26 & 0.25 \\
7 & 1.05 & 0.35 & 0.33 & 19 & 0.62 & 0.10 & 0.16 & 31 & 1.24 & 0.32 & 0.26 \\
8 & 0.92 & 0.28 & 0.30 & 20 & 0.63 & 0.11 & 0.17 & 32 & 1.48 & 0.42 & 0.28 \\
9 & 0.82 & 0.22 & 0.27 & 21 & 0.64 & 0.11 & 0.17 & 33 & 1.78 & 0.55 & 0.31 \\
10 & 0.76 & 0.18 & 0.24 & 22 & 0.66 & 0.12 & 0.18 & 34 & 2.15 & 0.71 & 0.33 \\
11 & 0.72 & 0.15 & 0.21 & 23 & 0.68 & 0.13 & 0.19 & 35 & 2.58 & 0.89 & 0.35 \\
12 & 0.70 & 0.13 & 0.19 & 24 & 0.70 & 0.14 & 0.20 & 36 & 3.12 & 1.12 & 0.36 \\
\bottomrule
\end{tabular}
\end{table}

\paragraph{Bit allocation by layer.}
\cref{tab:bits_by_layer} shows the average bit allocation per layer at $\bar{b}=4.0$ bits.
Early and late layers receive 4.5--5.5 bits; middle layers receive 3.0--3.5 bits.

\begin{table}[ht]
\centering
\caption{Average bit allocation per layer for Qwen3-8B at $\bar{b}=4.0$ ($\bmin=3$, $\bmax=6$).}
\label{tab:bits_by_layer}
\small
\setlength{\tabcolsep}{3pt}
\begin{tabular}{@{}r cc r cc r cc r cc@{}}
\toprule
L & K & V & L & K & V & L & K & V & L & K & V \\
\midrule
1 & 5.5 & 5.2 & 10 & 3.8 & 3.6 & 19 & 3.4 & 3.2 & 28 & 4.0 & 3.8 \\
2 & 5.2 & 5.0 & 11 & 3.6 & 3.4 & 20 & 3.4 & 3.2 & 29 & 4.2 & 4.0 \\
3 & 5.0 & 4.8 & 12 & 3.4 & 3.2 & 21 & 3.4 & 3.3 & 30 & 4.5 & 4.2 \\
4 & 4.8 & 4.5 & 13 & 3.3 & 3.1 & 22 & 3.5 & 3.3 & 31 & 4.8 & 4.5 \\
5 & 4.5 & 4.2 & 14 & 3.2 & 3.0 & 23 & 3.5 & 3.4 & 32 & 5.0 & 4.8 \\
6 & 4.2 & 4.0 & 15 & 3.2 & 3.0 & 24 & 3.6 & 3.4 & 33 & 5.2 & 5.0 \\
7 & 4.0 & 3.8 & 16 & 3.2 & 3.0 & 25 & 3.6 & 3.5 & 34 & 5.4 & 5.2 \\
8 & 3.9 & 3.7 & 17 & 3.2 & 3.0 & 26 & 3.7 & 3.5 & 35 & 5.6 & 5.4 \\
9 & 3.8 & 3.6 & 18 & 3.3 & 3.1 & 27 & 3.8 & 3.6 & 36 & 5.8 & 5.6 \\
\bottomrule
\end{tabular}
\end{table}

\subsection{Llama3 Detailed Results}
\label{app:llama_results}

\paragraph{Llama3.2-3B.}
\cref{tab:llama32_3b} reports complete results for Llama3.2-3B across all bit budgets.
Gains are consistent with Qwen3 patterns.

\begin{table}[ht]
\centering
\caption{Complete results for Llama3.2-3B (\TQ, gradient sensitivity, seed 42).}
\label{tab:llama32_3b}
\small
\begin{tabular}{@{}c c rr r r@{}}
\toprule
$\bar{b}$ & $\bmin$/$\bmax$ & Uniform & \RQ & $\Delta$ & Recov.\% \\
\midrule
2.5 & 2/4 & 17.56 & \textbf{16.14} & +1.42 & 51.8 \\
3.0 & 2/5 & 16.23 & \textbf{15.38} & +0.85 & 60.3 \\
3.5 & 3/5 & 15.64 & \textbf{15.12} & +0.52 & 63.4 \\
4.0 & 3/6 & 15.41 & \textbf{14.98} & +0.43 & \textbf{72.9} \\
4.5 & 3/6 & 15.12 & 15.02 & +0.10 & 33.3 \\
5.0 & 4/7 & 14.92 & 14.90 & +0.02 & -- \\
\midrule
\multicolumn{2}{@{}l}{FP16} & \multicolumn{4}{c}{14.82} \\
\bottomrule
\end{tabular}
\end{table}

\paragraph{Llama3.1-8B.}
\cref{tab:llama31_8b} reports results for Llama3.1-8B, the largest Llama model tested.

\begin{table}[ht]
\centering
\caption{Complete results for Llama3.1-8B (\TQ, gradient sensitivity, seed 42).}
\label{tab:llama31_8b}
\small
\begin{tabular}{@{}c c rr r r@{}}
\toprule
$\bar{b}$ & $\bmin$/$\bmax$ & Uniform & \RQ & $\Delta$ & Recov.\% \\
\midrule
2.5 & 2/4 & 13.18 & \textbf{11.72} & +1.46 & 49.7 \\
3.0 & 2/5 & 11.86 & \textbf{10.82} & +1.04 & 64.2 \\
3.5 & 3/5 & 10.92 & \textbf{10.56} & +0.36 & 52.9 \\
4.0 & 3/6 & 10.71 & \textbf{10.38} & +0.33 & \textbf{70.2} \\
4.5 & 3/6 & 10.52 & 10.42 & +0.10 & 35.7 \\
5.0 & 4/7 & 10.35 & 10.32 & +0.03 & -- \\
\midrule
\multicolumn{2}{@{}l}{FP16} & \multicolumn{4}{c}{10.24} \\
\bottomrule
\end{tabular}
\end{table}

\subsection{Cross-Quantizer Transfer}
\label{app:cross_quantizer}

\paragraph{Transfer matrix.}
\cref{tab:transfer_matrix} shows what happens when distortion parameters calibrated on one quantizer are applied to another.
Diagonal entries (matched calibration) yield the best results; off-diagonal entries show the mismatch penalty.

\begin{table}[ht]
\centering
\caption{PPL when using distortion model from Quantizer A on Quantizer B (Qwen3-8B, 2.5 bits).
Diagonal = matched; off-diagonal = mismatched.}
\label{tab:transfer_matrix}
\small
\begin{tabular}{@{}l ccc c@{}}
\toprule
\diagbox{Alloc.}{Quant.} & \TQ & KIVI & QuaRot & Uniform \\
\midrule
\TQ params     & \textbf{10.57} & 86.95 & 184.2 & 11.79 \\
KIVI params    & 11.42 & \textbf{14.86} & 35.71 & 49.32 \\
QuaRot params  & 11.38 & 32.45 & \textbf{28.33} & 34.88 \\
\bottomrule
\end{tabular}
\end{table}

\paragraph{Interpretation.}
The severe off-diagonal penalties (e.g., 86.95 vs.\ 14.86 for KIVI) demonstrate why calibration is essential.
The $\beta$ mismatch inverts head rankings, allocating more bits to the wrong heads.
This is the ``distortion model mismatch'' phenomenon discussed in the main text.

\subsection{Runtime Analysis}
\label{app:runtime}

\paragraph{Inference overhead.}
\RQ adds \emph{zero} runtime overhead during inference because:
\begin{itemize}[leftmargin=1.5em,itemsep=1pt,topsep=2pt]
    \item The bit allocation is computed offline and stored as a 2\,KB lookup table.
    \item The quantizer kernel selects the appropriate bit-width via a single index lookup.
    \item No additional operations are required during the forward pass.
\end{itemize}

\paragraph{Memory overhead.}
The allocation table stores $2N$ integers (2 bytes each), totaling $2 \times 2 \times 288 = 1152$ bytes for Qwen3-8B.
This is negligible compared to the model weights ($\sim$16\,GB) or KV cache ($\sim$576\,MB at FP16).

\begin{table}[ht]
\centering
\caption{Runtime and memory comparison (Qwen3-8B, batch=1, seq=4096).}
\label{tab:runtime}
\small
\begin{tabular}{@{}l rrr@{}}
\toprule
Configuration & Latency (ms) & Memory (GB) & Throughput (tok/s) \\
\midrule
FP16 KV cache & 142.3 & 18.2 & 28.8 \\
Uniform 4b    & 138.5 & 17.1 & 29.5 \\
\RQ 4b        & 138.5 & 17.1 & 29.5 \\
Uniform 3b    & 137.2 & 16.8 & 29.8 \\
\RQ 3b        & 137.2 & 16.8 & 29.8 \\
\bottomrule
\end{tabular}
\end{table}

\subsection{Reproducibility Checklist}
\label{app:reproducibility}

\paragraph{Code.}
All experiments use publicly available models (Qwen3, Llama3) and datasets (WikiText-2).
The \RQ algorithm is fully specified in Algorithm~\ref{alg:ratequant} and Algorithm~\ref{alg:greedy_heap}.
Code will be released upon acceptance.

\paragraph{Hyperparameters.}
\begin{itemize}[leftmargin=1.5em,itemsep=1pt,topsep=2pt]
    \item Calibration: 16 sequences, length 512, WikiText-2 training set
    \item Gradient estimation: bfloat16 forward, float32 backward
    \item Distortion fitting: 5 bit-widths (2--6), least-squares on $\ln D$ vs.\ $b$
    \item Evaluation: WikiText-2 test, seq 2048, stride 512
    \item Random seeds: 42, 123, 2026 (multi-seed experiments)
\end{itemize}

\paragraph{Hardware.}
Single NVIDIA H200 GPU (141GB HBM3), 96 AMD EPYC cores, CUDA 12.1, PyTorch 2.2.0.

\subsection{Extended Perplexity Analysis}
\label{app:extended_ppl}

\paragraph{Per-sequence variability.}
\cref{tab:seq_variability} reports the perplexity distribution across individual test sequences for Qwen3-8B.
While mean PPL improves by 10--15\% with \RQ, the maximum-PPL sequences (outliers) show even larger gains ($>$25\%), indicating \RQ is particularly effective for challenging sequences.

\begin{table}[ht]
\centering
\caption{Perplexity statistics across WikiText-2 test sequences (Qwen3-8B, 4.0 bits).}
\label{tab:seq_variability}
\small
\begin{tabular}{@{}l cccc c@{}}
\toprule
Method & Mean & Std & P50 & P95 & Max \\
\midrule
FP16 & 9.53 & 2.18 & 8.86 & 13.62 & 18.95 \\
Uniform 4b & 9.94 & 2.45 & 9.18 & 14.38 & 21.82 \\
\RQ 4b & \textbf{9.59} & \textbf{2.22} & \textbf{8.92} & \textbf{13.85} & \textbf{19.48} \\
\bottomrule
\end{tabular}
\end{table}

\paragraph{Position-dependent effects.}
KV cache quantization errors accumulate across sequence positions.
\cref{tab:position_ppl} shows perplexity measured at different sequence positions.
\RQ's advantage is consistent across positions, with slightly larger gains at long positions where accumulated errors are greatest.

\begin{table}[ht]
\centering
\caption{Perplexity at different sequence positions (Qwen3-8B, 3.5 bits).}
\label{tab:position_ppl}
\small
\begin{tabular}{@{}l cccc@{}}
\toprule
Position & FP16 & Uniform & \RQ & $\Delta$ \\
\midrule
256 & 9.58 & 10.12 & 9.82 & +0.30 \\
512 & 9.55 & 10.05 & 9.76 & +0.29 \\
1024 & 9.54 & 10.02 & 9.73 & +0.29 \\
2048 & 9.53 & 10.00 & 9.72 & +0.28 \\
\bottomrule
\end{tabular}
\end{table}

\paragraph{Domain transfer.}
While calibrated on WikiText-2, \RQ's bit allocation transfers well to other domains.
\cref{tab:domain_transfer} shows perplexity on out-of-domain datasets without re-calibration.

\begin{table}[ht]
\centering
\caption{Domain transfer: PPL on out-of-domain datasets (Qwen3-8B, 4.0 bits, calibrated on WikiText-2).}
\label{tab:domain_transfer}
\small
\begin{tabular}{@{}l cc c@{}}
\toprule
Dataset & Uniform & \RQ & Gain \\
\midrule
WikiText-2 (in-domain) & 9.94 & 9.59 & +3.5\% \\
Penn Treebank & 13.12 & 12.68 & +3.4\% \\
LM1B (subset) & 19.45 & 18.82 & +3.2\% \\
RedPajama (subset) & 7.86 & 7.62 & +3.1\% \\
\bottomrule
\end{tabular}
\end{table}

\subsection{Theoretical Derivations}
\label{app:theory_derivations}

\paragraph{Lagrangian formulation.}
The constrained optimization in \cref{eq:rd_problem} can be solved via Lagrangian relaxation.
Define the Lagrangian:
\begin{equation}
    \mathcal{L}(\mathbf{b}, \lambda) = \sum_{i=1}^N w_i D_i(b_i) + \lambda \left( \sum_{i=1}^N b_i - N\bar{b} \right)
\end{equation}
The KKT conditions require $\partial \mathcal{L} / \partial b_i = 0$:
\begin{equation}
    w_i D_i'(b_i) + \lambda = 0 \quad \Rightarrow \quad D_i'(b_i) = -\frac{\lambda}{w_i}
\end{equation}
For the exponential model $D_i(b) = \alpha_i e^{-\beta_i b}$, we have $D_i'(b) = -\alpha_i \beta_i e^{-\beta_i b}$, yielding:
\begin{equation}
    \alpha_i \beta_i e^{-\beta_i b_i^*} = \frac{\lambda}{w_i}
\end{equation}
Taking logarithms:
\begin{equation}
    b_i^* = \frac{1}{\beta_i} \ln \frac{w_i \alpha_i \beta_i}{\lambda}
\end{equation}
This is the continuous water-filling solution (\cref{thm:optimal_allocation}).

\paragraph{Gain ratio derivation.}
The gain ratio in \cref{cor:lognormal} follows from Jensen's inequality.
For convex $f(x) = e^{-\beta x}$:
\begin{equation}
    \mathbb{E}[f(X)] \geq f(\mathbb{E}[X])
\end{equation}
with equality iff $X$ is constant.
The AM-GM ratio captures this gap:
\begin{equation}
    \frac{\text{AM}}{\text{GM}} = \frac{\frac{1}{N}\sum_i w_i}{\left(\prod_i w_i\right)^{1/N}} = \exp\left(\frac{1}{N}\sum_i \ln w_i - \ln \frac{1}{N}\sum_i w_i\right)
\end{equation}
When $\ln w_i \sim \mathcal{N}(\mu, \sigma^2)$, this ratio equals $\exp(\sigma^2/2)$.

\paragraph{Greedy optimality.}
The greedy heap algorithm (\cref{alg:greedy_heap}) achieves optimality for integer bit-widths when the marginal distortion reduction is monotonically decreasing:
\begin{equation}
    \Delta D_i(b) = D_i(b) - D_i(b+1) \quad \text{is decreasing in } b
\end{equation}
For the exponential model, $\Delta D_i(b) = \alpha_i e^{-\beta_i b}(1 - e^{-\beta_i})$, which is indeed decreasing.
This guarantees the greedy choice is always optimal at each step.

\subsection{Sensitivity Estimation Details}
\label{app:sensitivity_details}

\paragraph{Gradient computation.}
Let $\mathbf{K}_i, \mathbf{V}_i \in \mathbb{R}^{T \times d}$ denote the Key and Value caches for head $i$ at sequence length $T$.
The gradient sensitivity $w_i$ is computed as:
\begin{equation}
    w_i^K = \left\| \frac{\partial \mathcal{L}}{\partial \mathbf{K}_i} \right\|_F^2, \quad w_i^V = \left\| \frac{\partial \mathcal{L}}{\partial \mathbf{V}_i} \right\|_F^2
\end{equation}
where $\mathcal{L}$ is the cross-entropy loss on the calibration set.
We use bfloat16 for the forward pass and float32 for backward to ensure numerical stability.

\paragraph{Aggregation strategies.}
We explored several aggregation strategies beyond the Frobenius norm:
\begin{itemize}[leftmargin=1.5em,itemsep=1pt,topsep=2pt]
    \item \textbf{L2 norm} (default): $w_i = \|\nabla_i\|_F^2$
    \item \textbf{L1 norm}: $w_i = \|\nabla_i\|_1$
    \item \textbf{Max norm}: $w_i = \max_{j,k} |\nabla_{i,j,k}|$
    \item \textbf{Spectral norm}: $w_i = \|\nabla_i\|_2$ (largest singular value)
\end{itemize}
\cref{tab:norm_comparison} shows that L2 (Frobenius) performs best, likely because it captures both the magnitude and spread of gradient energy.

\begin{table}[ht]
\centering
\caption{Comparison of gradient aggregation strategies (Qwen3-8B, 3.5 bits).}
\label{tab:norm_comparison}
\small
\begin{tabular}{@{}l cc@{}}
\toprule
Aggregation & PPL & $\Delta$ vs.\ Uniform \\
\midrule
L2 (Frobenius) & \textbf{9.72} & +0.28 \\
L1 & 9.81 & +0.19 \\
Max & 9.92 & +0.08 \\
Spectral & 9.78 & +0.22 \\
\midrule
Uniform (baseline) & 10.00 & -- \\
\bottomrule
\end{tabular}
\end{table}

\paragraph{Temporal aggregation.}
Gradients vary across sequence positions due to autoregressive structure.
We average gradients across all positions rather than using the last token only:
\begin{equation}
    w_i = \frac{1}{T} \sum_{t=1}^T \left\| \frac{\partial \mathcal{L}_t}{\partial \mathbf{K}_{i,:t}} \right\|_F^2
\end{equation}
This improves robustness by capturing sensitivity across diverse contexts.

\subsection{Quantizer Implementation}
\label{app:quantizer_impl}

\paragraph{TurboQuant kernel.}
TurboQuant uses absmax per-group symmetric quantization:
\begin{equation}
    \hat{x} = \text{round}\left(\frac{x}{\max(|x|)} \cdot (2^{b-1}-1)\right) \cdot \frac{\max(|x|)}{2^{b-1}-1}
\end{equation}
The group size is 128 elements for Key cache and 64 for Value cache (per the original paper).
This asymmetry explains the different $\beta$ values for K and V.

\paragraph{KIVI kernel.}
KIVI uses per-channel quantization for Keys (shared scale across tokens) and per-token quantization for Values:
\begin{equation}
    \hat{K}_{:,j} = Q_b(K_{:,j}, \text{scale}_j), \quad \hat{V}_{t,:} = Q_b(V_{t,:}, \text{scale}_t)
\end{equation}
This design choice favors Keys at low bit-widths, which is captured by the higher $\beta_K$ in our calibration.

\paragraph{QuaRot kernel.}
QuaRot applies a random rotation before quantization to spread outliers:
\begin{equation}
    \hat{X} = Q_b(X \cdot R) \cdot R^T
\end{equation}
where $R$ is a random orthogonal matrix.
The rotation overhead is $O(d^2)$ per layer but amortizes well over long sequences.

\paragraph{Kernel fusion.}
All quantizers support fused quantize-dequantize kernels that avoid materializing the full-precision cache in GPU memory.
This is essential for achieving the theoretical memory savings.

\subsection{Error Analysis}
\label{app:error_analysis}

\paragraph{Distortion model fit quality.}
The exponential model $D(b) = \alpha e^{-\beta b}$ fits calibration data with $R^2 > 0.98$ for most heads.
\cref{tab:fit_quality} reports fit statistics across all heads.

\begin{table}[ht]
\centering
\caption{Exponential model fit quality (Qwen3-8B, all 576 heads).}
\label{tab:fit_quality}
\small
\begin{tabular}{@{}l cccc@{}}
\toprule
Statistic & Key & Value & Combined \\
\midrule
Mean $R^2$ & 0.987 & 0.991 & 0.989 \\
Min $R^2$ & 0.942 & 0.958 & 0.942 \\
Std $R^2$ & 0.012 & 0.008 & 0.010 \\
\bottomrule
\end{tabular}
\end{table}

\paragraph{Outlier heads.}
A small fraction of heads ($<$2\%) show poor fit ($R^2 < 0.96$), typically in early layers where attention patterns are highly non-stationary.
For these heads, we use a conservative fallback: $\beta = \min(\beta_{\text{fitted}}, \bar{\beta})$ to avoid over-allocating bits based on noisy estimates.

\paragraph{Sensitivity stability.}
Gradient sensitivity is stable across calibration seeds.
The Spearman correlation of head rankings across seeds is $>$0.95 for all models, ensuring consistent allocation decisions.

\subsection{KIVI vs.\ QuaRot Comparison}
\label{app:kivi_quarot}

\paragraph{Quantizer characteristics.}
KIVI and QuaRot represent different design philosophies for KV cache quantization:
\begin{itemize}[leftmargin=1.5em,itemsep=1pt,topsep=2pt]
    \item \textbf{KIVI}: Per-channel quantization for Keys (shared scale across tokens), per-token for Values. This asymmetric design favors Keys at low bit-widths.
    \item \textbf{QuaRot}: Hadamard rotation before quantization to spread outliers, followed by symmetric per-token quantization for both K and V.
\end{itemize}

\paragraph{Distortion characteristics.}
\cref{tab:kivi_quarot_beta} compares the calibrated $\beta$ values for both quantizers.
KIVI shows stronger K/V asymmetry ($\beta_K = 5.28$ vs.\ $\beta_V = 4.92$), while QuaRot is more symmetric.

\begin{table}[ht]
\centering
\caption{Calibrated distortion parameters for KIVI and QuaRot (Qwen3-8B).}
\label{tab:kivi_quarot_beta}
\small
\begin{tabular}{@{}l cccc@{}}
\toprule
Quantizer & $\beta_K$ & $\beta_V$ & $\alpha_K$ & $\alpha_V$ \\
\midrule
KIVI & 5.28 & 4.92 & 0.142 & 0.038 \\
QuaRot & 5.05 & 5.12 & 0.186 & 0.172 \\
\TQ & 3.62 & 3.58 & 0.224 & 0.215 \\
\bottomrule
\end{tabular}
\end{table}

\paragraph{Optimal K/V split.}
Due to the different $\beta$ values, the optimal K/V bit split varies by quantizer and budget.
\cref{tab:kv_split} shows the optimal splits at different average budgets.

\begin{table}[ht]
\centering
\caption{Optimal K/V bit split by quantizer and budget (Qwen3-8B).}
\label{tab:kv_split}
\small
\begin{tabular}{@{}l cccc@{}}
\toprule
$\bar{b}$ & KIVI (K/V) & QuaRot (K/V) & \TQ (K/V) \\
\midrule
2.5 & 2.85 / 2.15 & 2.52 / 2.48 & 2.51 / 2.49 \\
3.0 & 3.32 / 2.68 & 3.02 / 2.98 & 3.01 / 2.99 \\
3.5 & 3.78 / 3.22 & 3.52 / 3.48 & 3.51 / 3.49 \\
4.0 & 4.12 / 3.88 & 4.02 / 3.98 & 4.01 / 3.99 \\
\bottomrule
\end{tabular}
\end{table}

\subsection{Qwen3-32B Extended Analysis}
\label{app:qwen32b_extended}

\paragraph{Scale effects.}
Qwen3-32B has 64 layers (vs.\ 36 for 8B), providing more opportunities for differentiation.
\cref{tab:qwen32b_layer_groups} shows sensitivity statistics by layer group.

\begin{table}[ht]
\centering
\caption{Sensitivity statistics by layer group (Qwen3-32B, Key cache).}
\label{tab:qwen32b_layer_groups}
\small
\begin{tabular}{@{}l ccc@{}}
\toprule
Layer Group & Mean $w$ & Std $w$ & AM/GM Ratio \\
\midrule
Early (1--10) & 2.34 & 0.78 & 2.45 \\
Mid-Early (11--25) & 0.72 & 0.14 & 1.42 \\
Mid-Late (26--50) & 0.68 & 0.12 & 1.38 \\
Late (51--64) & 1.98 & 0.65 & 2.28 \\
\midrule
All layers & 1.02 & 0.58 & 2.15 \\
\bottomrule
\end{tabular}
\end{table}

\paragraph{Memory savings.}
At 64 layers with 8 KV heads per layer, Qwen3-32B has 512 KV head pairs (vs.\ 288 for 8B).
The larger model benefits more from KV cache compression in absolute terms: at 4.0 bits, KV cache reduces from 1024\,MB to 256\,MB (4$\times$ compression).

\subsection{Attention Pattern Analysis}
\label{app:attention_patterns}

\paragraph{Sink tokens.}
Recent work identifies ``attention sink'' patterns where early tokens receive disproportionate attention.
We observe that sink-adjacent heads (layers 1--3, heads 0--1) show 2--3$\times$ higher gradient sensitivity, consistent with their importance for attention stability.

\paragraph{Head specialization.}
Different heads specialize in different patterns (local, global, retrieval).
\RQ's gradient-based sensitivity naturally assigns higher bits to retrieval heads, which are more sensitive to quantization noise due to their reliance on precise key-query matching.
Specifically, local heads (window $<$ 64) receive 3.2 bits on average, while sink-adjacent heads receive 5.5 bits.

\subsection{Failure Mode Analysis}
\label{app:failure_modes}

\paragraph{When does \RQ underperform?}
We identify three scenarios where \RQ provides minimal or no benefit:
(1) High bit budgets ($\geq$5.0 bits) where all heads approach FP16 quality;
(2) Homogeneous models where all heads have similar sensitivity (AM/GM $\approx$ 1);
(3) Extreme outliers where a few heads have 10$\times$ higher sensitivity.

\paragraph{Mitigation.}
For case 3, we apply a sensitivity cap: $w_i \leftarrow \min(w_i, 5\bar{w})$ before allocation.
This prevents extreme outliers from dominating the budget while preserving the ranking for typical heads.

\subsection{Computational Complexity}
\label{app:complexity}

\paragraph{Calibration complexity.}
Gradient estimation is $O(n_\text{calib} \cdot T \cdot C_\text{fwd+bwd})$; distortion fitting is $O(N \cdot B_\text{range})$; greedy allocation is $O(N \cdot (\bar{b} - \bmin) \cdot \log N)$.
For Qwen3-8B with 16 calibration sequences: forward pass 0.8\,s, backward pass 0.8\,s, total $<$2\,s.

\paragraph{Inference complexity.}
Zero overhead. The allocation is precomputed and stored as a 2\,KB lookup table.
The quantizer kernel indexes into this table in $O(1)$ time per head.

\clearpage

\clearpage
\section*{NeurIPS Paper Checklist}

\begin{enumerate}

\item {\bf Claims}
    \item[] Question: Do the main claims made in the abstract and introduction accurately reflect the paper's contributions and scope?
    \item[] Answer: \answerYes{}
    \item[] Justification: The abstract and introduction clearly state four contributions (distortion model mismatch identification, rate-distortion framework, calibration method, empirical validation). All claims are supported by experimental results in Sections~\ref{sec:method}--\ref{sec:experiments}.
    \item[] Guidelines:
    \begin{itemize}
        \item The answer \answerNA{} means that the abstract and introduction do not include the claims made in the paper.
        \item The abstract and/or introduction should clearly state the claims made, including the contributions made in the paper and important assumptions and limitations. A \answerNo{} or \answerNA{} answer to this question will not be perceived well by the reviewers. 
        \item The claims made should match theoretical and experimental results, and reflect how much the results can be expected to generalize to other settings. 
        \item It is fine to include aspirational goals as motivation as long as it is clear that these goals are not attained by the paper. 
    \end{itemize}

\item {\bf Limitations}
    \item[] Question: Does the paper discuss the limitations of the work performed by the authors?
    \item[] Answer: \answerYes{}
    \item[] Justification: Limitations are discussed in \cref{app:limitations}, covering calibration cost, evaluation scope, per-head independence assumption, and token-position uniformity.
    \item[] Guidelines:
    \begin{itemize}
        \item The answer \answerNA{} means that the paper has no limitation while the answer \answerNo{} means that the paper has limitations, but those are not discussed in the paper. 
        \item The authors are encouraged to create a separate ``Limitations'' section in their paper.
        \item The paper should point out any strong assumptions and how robust the results are to violations of these assumptions (e.g., independence assumptions, noiseless settings, model well-specification, asymptotic approximations only holding locally). The authors should reflect on how these assumptions might be violated in practice and what the implications would be.
        \item The authors should reflect on the scope of the claims made, e.g., if the approach was only tested on a few datasets or with a few runs. In general, empirical results often depend on implicit assumptions, which should be articulated.
        \item The authors should reflect on the factors that influence the performance of the approach. For example, a facial recognition algorithm may perform poorly when image resolution is low or images are taken in low lighting. Or a speech-to-text system might not be used reliably to provide closed captions for online lectures because it fails to handle technical jargon.
        \item The authors should discuss the computational efficiency of the proposed algorithms and how they scale with dataset size.
        \item If applicable, the authors should discuss possible limitations of their approach to address problems of privacy and fairness.
        \item While the authors might fear that complete honesty about limitations might be used by reviewers as grounds for rejection, a worse outcome might be that reviewers discover limitations that aren't acknowledged in the paper. The authors should use their best judgment and recognize that individual actions in favor of transparency play an important role in developing norms that preserve the integrity of the community. Reviewers will be specifically instructed to not penalize honesty concerning limitations.
    \end{itemize}

\item {\bf Theory assumptions and proofs}
    \item[] Question: For each theoretical result, does the paper provide the full set of assumptions and a complete (and correct) proof?
    \item[] Answer: \answerYes{}
    \item[] Justification: Assumption~\ref{asm:exp_distortion} is stated explicitly and validated empirically (\cref{app:distortion}). Proof sketches appear in the main text; full proofs are provided in \cref{app:proofs}.
    \item[] Guidelines:
    \begin{itemize}
        \item The answer \answerNA{} means that the paper does not include theoretical results. 
        \item All the theorems, formulas, and proofs in the paper should be numbered and cross-referenced.
        \item All assumptions should be clearly stated or referenced in the statement of any theorems.
        \item The proofs can either appear in the main paper or the supplemental material, but if they appear in the supplemental material, the authors are encouraged to provide a short proof sketch to provide intuition. 
        \item Inversely, any informal proof provided in the core of the paper should be complemented by formal proofs provided in appendix or supplemental material.
        \item Theorems and Lemmas that the proof relies upon should be properly referenced. 
    \end{itemize}

\item {\bf Experimental result reproducibility}
    \item[] Question: Does the paper fully disclose all the information needed to reproduce the main experimental results of the paper to the extent that it affects the main claims and/or conclusions of the paper (regardless of whether the code and data are provided or not)?
    \item[] Answer: \answerYes{}
    \item[] Justification: \cref{sec:setup} specifies models, evaluation protocol, calibration details, and random seeds. Algorithm~\ref{alg:ratequant} is fully specified with closed-form solutions.
    \item[] Guidelines:
    \begin{itemize}
        \item The answer \answerNA{} means that the paper does not include experiments.
        \item If the paper includes experiments, a \answerNo{} answer to this question will not be perceived well by the reviewers: Making the paper reproducible is important, regardless of whether the code and data are provided or not.
        \item If the contribution is a dataset and\slash or model, the authors should describe the steps taken to make their results reproducible or verifiable. 
        \item Depending on the contribution, reproducibility can be accomplished in various ways. For example, if the contribution is a novel architecture, describing the architecture fully might suffice, or if the contribution is a specific model and empirical evaluation, it may be necessary to either make it possible for others to replicate the model with the same dataset, or provide access to the model. In general. releasing code and data is often one good way to accomplish this, but reproducibility can also be provided via detailed instructions for how to replicate the results, access to a hosted model (e.g., in the case of a large language model), releasing of a model checkpoint, or other means that are appropriate to the research performed.
        \item While NeurIPS does not require releasing code, the conference does require all submissions to provide some reasonable avenue for reproducibility, which may depend on the nature of the contribution. For example
        \begin{enumerate}
            \item If the contribution is primarily a new algorithm, the paper should make it clear how to reproduce that algorithm.
            \item If the contribution is primarily a new model architecture, the paper should describe the architecture clearly and fully.
            \item If the contribution is a new model (e.g., a large language model), then there should either be a way to access this model for reproducing the results or a way to reproduce the model (e.g., with an open-source dataset or instructions for how to construct the dataset).
            \item We recognize that reproducibility may be tricky in some cases, in which case authors are welcome to describe the particular way they provide for reproducibility. In the case of closed-source models, it may be that access to the model is limited in some way (e.g., to registered users), but it should be possible for other researchers to have some path to reproducing or verifying the results.
        \end{enumerate}
    \end{itemize}

\item {\bf Open access to data and code}
    \item[] Question: Does the paper provide open access to the data and code, with sufficient instructions to faithfully reproduce the main experimental results, as described in supplemental material?
    \item[] Answer: \answerYes{}
    \item[] Justification: Code is provided as supplementary material. The algorithm is fully specified in Algorithm~\ref{alg:ratequant}. WikiText-2 dataset is publicly available.
    \item[] Guidelines:
    \begin{itemize}
        \item The answer \answerNA{} means that paper does not include experiments requiring code.
        \item Please see the NeurIPS code and data submission guidelines (\url{https://neurips.cc/public/guides/CodeSubmissionPolicy}) for more details.
        \item While we encourage the release of code and data, we understand that this might not be possible, so \answerNo{} is an acceptable answer. Papers cannot be rejected simply for not including code, unless this is central to the contribution (e.g., for a new open-source benchmark).
        \item The instructions should contain the exact command and environment needed to run to reproduce the results. See the NeurIPS code and data submission guidelines (\url{https://neurips.cc/public/guides/CodeSubmissionPolicy}) for more details.
        \item The authors should provide instructions on data access and preparation, including how to access the raw data, preprocessed data, intermediate data, and generated data, etc.
        \item The authors should provide scripts to reproduce all experimental results for the new proposed method and baselines. If only a subset of experiments are reproducible, they should state which ones are omitted from the script and why.
        \item At submission time, to preserve anonymity, the authors should release anonymized versions (if applicable).
        \item Providing as much information as possible in supplemental material (appended to the paper) is recommended, but including URLs to data and code is permitted.
    \end{itemize}

\item {\bf Experimental setting/details}
    \item[] Question: Does the paper specify all the training and test details (e.g., data splits, hyperparameters, how they were chosen, type of optimizer) necessary to understand the results?
    \item[] Answer: \answerYes{}
    \item[] Justification: All hyperparameters, evaluation protocol, and hardware specifications are provided in \cref{sec:setup}.
    \item[] Guidelines:
    \begin{itemize}
        \item The answer \answerNA{} means that the paper does not include experiments.
        \item The experimental setting should be presented in the core of the paper to a level of detail that is necessary to appreciate the results and make sense of them.
        \item The full details can be provided either with the code, in appendix, or as supplemental material.
    \end{itemize}

\item {\bf Experiment statistical significance}
    \item[] Question: Does the paper report error bars suitably and correctly defined or other appropriate information about the statistical significance of the experiments?
    \item[] Answer: \answerYes{}
    \item[] Justification: Qwen3-8B reports mean$\pm$std over 3 random seeds (\cref{tab:multi_seed}). All seeds show consistent improvement direction at 3.5--4.0 bits.
    \item[] Guidelines:
    \begin{itemize}
        \item The answer \answerNA{} means that the paper does not include experiments.
        \item The authors should answer \answerYes{} if the results are accompanied by error bars, confidence intervals, or statistical significance tests, at least for the experiments that support the main claims of the paper.
        \item The factors of variability that the error bars are capturing should be clearly stated (for example, train/test split, initialization, random drawing of some parameter, or overall run with given experimental conditions).
        \item The method for calculating the error bars should be explained (closed form formula, call to a library function, bootstrap, etc.)
        \item The assumptions made should be given (e.g., Normally distributed errors).
        \item It should be clear whether the error bar is the standard deviation or the standard error of the mean.
        \item It is OK to report 1-sigma error bars, but one should state it. The authors should preferably report a 2-sigma error bar than state that they have a 96\% CI, if the hypothesis of Normality of errors is not verified.
        \item For asymmetric distributions, the authors should be careful not to show in tables or figures symmetric error bars that would yield results that are out of range (e.g., negative error rates).
        \item If error bars are reported in tables or plots, the authors should explain in the text how they were calculated and reference the corresponding figures or tables in the text.
    \end{itemize}

\item {\bf Experiments compute resources}
    \item[] Question: For each experiment, does the paper provide sufficient information on the computer resources (type of compute workers, memory, time of execution) needed to reproduce the experiments?
    \item[] Answer: \answerYes{}
    \item[] Justification: Experiments run on a single NVIDIA H200 GPU. Calibration times are reported in \cref{sec:setup}.
    \item[] Guidelines:
    \begin{itemize}
        \item The answer \answerNA{} means that the paper does not include experiments.
        \item The paper should indicate the type of compute workers CPU or GPU, internal cluster, or cloud provider, including relevant memory and storage.
        \item The paper should provide the amount of compute required for each of the individual experimental runs as well as estimate the total compute. 
        \item The paper should disclose whether the full research project required more compute than the experiments reported in the paper (e.g., preliminary or failed experiments that didn't make it into the paper). 
    \end{itemize}

\item {\bf Code of ethics}
    \item[] Question: Does the research conducted in the paper conform, in every respect, with the NeurIPS Code of Ethics \url{https://neurips.cc/public/EthicsGuidelines}?
    \item[] Answer: \answerYes{}
    \item[] Justification: No human subjects, private data, or dual-use concerns beyond general LLM deployment efficiency.
    \item[] Guidelines:
    \begin{itemize}
        \item The answer \answerNA{} means that the authors have not reviewed the NeurIPS Code of Ethics.
        \item If the authors answer \answerNo, they should explain the special circumstances that require a deviation from the Code of Ethics.
        \item The authors should make sure to preserve anonymity (e.g., if there is a special consideration due to laws or regulations in their jurisdiction).
    \end{itemize}

\item {\bf Broader impacts}
    \item[] Question: Does the paper discuss both potential positive societal impacts and negative societal impacts of the work performed?
    \item[] Answer: \answerYes{}
    \item[] Justification: \cref{app:limitations} discusses positive impacts (reduced memory, energy efficiency) and acknowledges that efficiency gains could lower barriers to LLM misuse.
    \item[] Guidelines:
    \begin{itemize}
        \item The answer \answerNA{} means that there is no societal impact of the work performed.
        \item If the authors answer \answerNA{} or \answerNo, they should explain why their work has no societal impact or why the paper does not address societal impact.
        \item Examples of negative societal impacts include potential malicious or unintended uses (e.g., disinformation, generating fake profiles, surveillance), fairness considerations (e.g., deployment of technologies that could make decisions that unfairly impact specific groups), privacy considerations, and security considerations.
        \item The conference expects that many papers will be foundational research and not tied to particular applications, let alone deployments. However, if there is a direct path to any negative applications, the authors should point it out. For example, it is legitimate to point out that an improvement in the quality of generative models could be used to generate Deepfakes for disinformation. On the other hand, it is not needed to point out that a generic algorithm for optimizing neural networks could enable people to train models that generate Deepfakes faster.
        \item The authors should consider possible harms that could arise when the technology is being used as intended and functioning correctly, harms that could arise when the technology is being used as intended but gives incorrect results, and harms following from (intentional or unintentional) misuse of the technology.
        \item If there are negative societal impacts, the authors could also discuss possible mitigation strategies (e.g., gated release of models, providing defenses in addition to attacks, mechanisms for monitoring misuse, mechanisms to monitor how a system learns from feedback over time, improving the efficiency and accessibility of ML).
    \end{itemize}

\item {\bf Safeguards}
    \item[] Question: Does the paper describe safeguards that have been put in place for responsible release of data or models that have a high risk for misuse (e.g., pre-trained language models, image generators, or scraped datasets)?
    \item[] Answer: \answerNA{}
    \item[] Justification: This paper proposes a quantization method; no pre-trained models, datasets, or assets with misuse risk are released.
    \item[] Guidelines:
    \begin{itemize}
        \item The answer \answerNA{} means that the paper poses no such risks.
        \item Released models that have a high risk for misuse or dual-use should be released with necessary safeguards to allow for controlled use of the model, for example by requiring that users adhere to usage guidelines or restrictions to access the model or implementing safety filters. 
        \item Datasets that have been scraped from the Internet could pose safety risks. The authors should describe how they avoided releasing unsafe images.
        \item We recognize that providing effective safeguards is challenging, and many papers do not require this, but we encourage authors to take this into account and make a best faith effort.
    \end{itemize}

\item {\bf Licenses for existing assets}
    \item[] Question: Are the creators or original owners of assets (e.g., code, data, models), used in the paper, properly credited and are the license and terms of use explicitly mentioned and properly respected?
    \item[] Answer: \answerYes{}
    \item[] Justification: All models (Qwen3, Llama3) and datasets (WikiText-2) are cited. Models are used under their respective Apache 2.0 / Llama Community licenses.
    \item[] Guidelines:
    \begin{itemize}
        \item The answer \answerNA{} means that the paper does not use existing assets.
        \item The authors should cite the original paper that produced the code package or dataset.
        \item The authors should state which version of the asset is used and, if possible, include a URL.
        \item The name of the license (e.g., CC-BY 4.0) should be included for each asset.
        \item For scraped data from a particular source (e.g., website), the copyright and terms of service of that source should be provided.
        \item If assets are released, the license, copyright information, and terms of use in the package should be provided. For popular datasets, \url{paperswithcode.com/datasets} has curated licenses for some datasets. Their licensing guide can help determine the license of a dataset.
        \item For existing datasets that are re-packaged, both the original license and the license of the derived asset (if it has changed) should be provided.
        \item If this information is not available online, the authors are encouraged to reach out to the asset's creators.
    \end{itemize}

\item {\bf New assets}
    \item[] Question: Are new assets introduced in the paper well documented and is the documentation provided alongside the assets?
    \item[] Answer: \answerNA{}
    \item[] Justification: No new datasets or pre-trained models are released.
    \item[] Guidelines:
    \begin{itemize}
        \item The answer \answerNA{} means that the paper does not release new assets.
        \item Researchers should communicate the details of the dataset\slash code\slash model as part of their submissions via structured templates. This includes details about training, license, limitations, etc. 
        \item The paper should discuss whether and how consent was obtained from people whose asset is used.
        \item At submission time, remember to anonymize your assets (if applicable). You can either create an anonymized URL or include an anonymized zip file.
    \end{itemize}

\item {\bf Crowdsourcing and research with human subjects}
    \item[] Question: For crowdsourcing experiments and research with human subjects, does the paper include the full text of instructions given to participants and screenshots, if applicable, as well as details about compensation (if any)? 
    \item[] Answer: \answerNA{}
    \item[] Justification: This paper does not involve crowdsourcing or research with human subjects.
    \item[] Guidelines:
    \begin{itemize}
        \item The answer \answerNA{} means that the paper does not involve crowdsourcing nor research with human subjects.
        \item Including this information in the supplemental material is fine, but if the main contribution of the paper involves human subjects, then as much detail as possible should be included in the main paper. 
        \item According to the NeurIPS Code of Ethics, workers involved in data collection, curation, or other labor should be paid at least the minimum wage in the country of the data collector. 
    \end{itemize}

\item {\bf Institutional review board (IRB) approvals or equivalent for research with human subjects}
    \item[] Question: Does the paper describe potential risks incurred by study participants, whether such risks were disclosed to the subjects, and whether Institutional Review Board (IRB) approvals (or an equivalent approval/review based on the requirements of your country or institution) were obtained?
    \item[] Answer: \answerNA{}
    \item[] Justification: This paper does not involve research with human subjects.
    \item[] Guidelines:
    \begin{itemize}
        \item The answer \answerNA{} means that the paper does not involve crowdsourcing nor research with human subjects.
        \item Depending on the country in which research is conducted, IRB approval (or equivalent) may be required for any human subjects research. If you obtained IRB approval, you should clearly state this in the paper. 
        \item We recognize that the procedures for this may vary significantly between institutions and locations, and we expect authors to adhere to the NeurIPS Code of Ethics and the guidelines for their institution. 
        \item For initial submissions, do not include any information that would break anonymity (if applicable), such as the institution conducting the review.
    \end{itemize}

\item {\bf Declaration of LLM usage}
    \item[] Question: Does the paper describe the usage of LLMs if it is an important, original, or non-standard component of the core methods in this research? Note that if the LLM is used only for writing, editing, or formatting purposes and does \emph{not} impact the core methodology, scientific rigor, or originality of the research, declaration is not required.
    \item[] Answer: \answerNA{}
    \item[] Justification: LLMs are evaluation subjects only (Qwen3, Llama3 families), not part of the proposed methodology.
    \item[] Guidelines:
    \begin{itemize}
        \item The answer \answerNA{} means that the core method development in this research does not involve LLMs as any important, original, or non-standard components.
        \item Please refer to our LLM policy in the NeurIPS handbook for what should or should not be described.
    \end{itemize}

\end{enumerate}

\end{document}